\documentclass{article}

\usepackage{lineno,hyperref}

\modulolinenumbers[5]

\usepackage[utf8]{inputenc}
\usepackage[T1]{fontenc}
\usepackage{appendix}

\usepackage{amsmath}
\usepackage{amssymb}
\usepackage{mathrsfs} 
\usepackage{bbm} 
\usepackage{color}
\usepackage[svgnames]{xcolor}
\usepackage{graphics}
\usepackage{graphicx}
\usepackage{tikz} 
\usetikzlibrary{shapes.misc,automata,positioning,arrows}
\usepgflibrary{shapes.arrows} 
\usepackage[ruled,linesnumbered,vlined,scleft]{algorithm2e}
\usepackage{url}
\usepackage{comment}
\usepackage{theorem}
\usepackage{stmaryrd}

\usepackage{thmtools}
\usepackage{thm-restate}
\usepackage{cleveref}

\newcommand{\ie}{\mbox{i.e.}}
\newcommand{\eg}{\mbox{e.g.}}

\newcommand{\wrt}{\mbox{w.r.t.}}
\newcommand{\st}{\mbox{s.t.}}

\newcommand{\etal}{\mbox{et al}}
\newcommand{\myskip}{\smallskip} 

\newcommand{\defined}{\mathrel{\mathop:}=}
\newcommand{\grammar}{\mathrel{\mathop:}\defined}

\newtheorem{definition}{Definition}[section]
\newtheorem{observation}{Observation}
{\theorembodyfont{\upshape}
\newtheorem{example}{Example}[section]}
\newenvironment{proof}{\noindent\textbf{Proof:}\null\hfill\null\newline}{\hfill\qed}

\declaretheorem[name=Lemma,numberwithin=section]{lemma}
\declaretheorem[name=Proposition,numberwithin=section]{proposition}

\newcommand{\ps}[1]{\mathscr{P}(#1)} 

\newcommand{\Prp}{\ensuremath{\mathcal{P}}}
\newcommand{\Lang}{\ensuremath{\mathcal{L}}} 

\newcommand{\KB}{\ensuremath{\mathcal{KB}}} 

\newcommand{\U}{\ensuremath{\mathcal{U}}} 

\newcommand{\limp}{\rightarrow} 
\newcommand{\liff}{\leftrightarrow} 

\newcommand{\sat}{\Vdash}

\newcommand{\entails}{\models}

\newcommand{\twiddle}{\mathrel|\joinrel\sim}
\newcommand{\ntwiddle}{\not\twiddle}
\newcommand{\typ}{\bullet}
\newcommand{\typup}[1]{#1^{\typ}}
\newcommand{\Typ}{\mathbf{T}}

\newcommand{\RM}{\mathscr{R}} 
\newcommand{\Mod}[1]{{\ensuremath{\text{\it Mod}(#1)}}} 
\newcommand{\Cn}[2]{{\ensuremath{\text{\it Cn}_{#1}(#2)}}} 
\newcommand{\pref}{\prec} 

\newcommand{\states}[2]{\llbracket#1\rrbracket^{#2}} 
\newcommand{\rk}{\ensuremath{\text{\it rk}}} 

\newcommand{\mpref}[1]{\trianglelefteq_{#1}}
\newcommand{\mprefstrict}[1]{\triangleleft_{#1}}

\newcommand{\NMentails}{\mathrel|\joinrel\approx}
\newcommand{\nNMentails}{\not\NMentails}

\newcommand{\p}{\ensuremath{\mathsf{p}}}
\newcommand{\bird}{\ensuremath{\mathsf{b}}}
\newcommand{\fly}{\ensuremath{\mathsf{f}}}
\newcommand{\ost}{\ensuremath{\mathsf{o}}}

\newcommand{\rob}{\ensuremath{\mathsf{r}}}

\newcommand{\LM}{\mathrm{LM}}
\newcommand{\PT}{\mathrm{PT}}

\newcommand{\imp}{\rightarrow}

\newcommand{\qed}{\ensuremath{\boxempty}}

\usepackage{authblk}

\title{On Rational Entailment for Propositional Typicality Logic\footnote{This is a preprint version of a paper accepted for publication in  \emph{Artificial Intelligence}, DOI:\url{https://doi.org/10.1016/j.artint.2019.103178}.}}

\author[1]{Richard Booth}
\author[2]{Giovanni Casini}
\author[3]{Thomas Meyer}
\author[4]{Ivan Varzinczak}

\affil[1]{\small Cardiff University, United Kingdom. Email: \emph{{BoothR2@cardiff.ac.uk}}}
\affil[2]{CSC, University of Luxembourg, Luxembourg. Email: \emph{{giovanni.casini@uni.lu}}}
\affil[3]{University of Cape Town and CAIR, South Africa. Email: \emph{{tmeyer@cs.uct.ac.za}}}
\affil[4]{CRIL, Univ.\ Artois \& CNRS, France. Email: \emph{{varzinczak@cril.fr}}}

\date{}

\begin{document}

\maketitle

\begin{abstract}
Propositional Typicality Logic (PTL) is a recently proposed logic,  obtained by enriching classical propositional logic with a typicality operator capturing the most typical (alias normal or conventional) situations in which a given sentence holds. The semantics of PTL is in terms of ranked models as studied in the well-known KLM approach to preferential reasoning and therefore KLM-style rational consequence relations can be embedded in PTL. In spite of the non-monotonic features introduced by the semantics adopted for the typicality operator, the obvious Tarskian definition of entailment for PTL remains monotonic and is therefore not appropriate in many contexts. Our first important result is an impossibility theorem showing that a set of proposed postulates that at first all seem appropriate for a notion of entailment with regard to typicality cannot be satisfied simultaneously. Closer inspection reveals that this result is best interpreted as an argument for advocating the development of more than one type of PTL entailment. In the spirit of this interpretation, we investigate three different (semantic) versions of entailment for PTL, each one based on the definition of rational closure as introduced by Lehmann and Magidor for KLM-style conditionals, and constructed using different notions of minimality.
\end{abstract}




\section{Introduction}\label{Introduction}

Propositional Typicality Logic (PTL)~\cite{BoothEtAl2012,BoothEtAl2013} is a recently proposed logic allowing for the representation of and reasoning with an explicit notion of {\em typicality}. It is obtained by enriching classical propositional logic with a {\em typicality operator}~$\typ$, the intuition of which is to refer to those most typical (or normal or conventional) situations in which a given sentence holds. PTL is characterised using a preferential semantics similar to that originally proposed by Shoham~\cite{Shoham1988} and extensively developed by Kraus~\etal.~\cite{KrausEtAl1990} and Lehmann and Magidor~\cite{LehmannMagidor1992} in the propositional case, with close connections to the formalisms developed by Pearl and Goldszmidt \cite{Pearl1990,PearlEtAl1990}, 
and by others~\cite{Boutilier1994,BritzEtAl2011b,BritzEtAl2011c,GiordanoEtAl2013,QuantzRoyer1992,CasiniStraccia2010,CasiniStraccia2011,CasiniStraccia2013} in more expressive languages.

In spite of the non-monotonic features introduced by the adoption of a preferential semantics for~$\typ$, the obvious definition of entailment for~PTL, \ie, the one based on a Tarskian notion of logical consequence, remains monotonic. Of course, such a notion of entailment is inappropriate in non-monotonic contexts, in particular when reasoning about typicality, as is already clear from an enriched version of the classical Tweety example: If birds typically fly, and penguins are birds (and that is all we know), we would expect to be able to conclude that typical penguins are typical birds, and therefore that typical penguins  fly. Learning that penguins typically do not fly should lead us to conclude that penguins are not typical birds, and  to retract the conclusions about typical penguins being typical birds, and about typical penguins flying.

In this paper, we investigate three semantic versions of entailment for~PTL, constructed using three different forms of minimality. All these are based on the notion of rational closure as defined by Lehmann and Magidor~\cite{LehmannMagidor1992} for KLM-style conditionals in a propositional setting. We show that they can be viewed as distinct extensions of rational closure, equivalent  with respect to the conditional language originally proposed by Kraus~\etal., but different in the~PTL framework.

We shall study the aforementioned forms of entailment in an abstract formal setting, obtained by proposing a set of postulates that, at first glance, seem appropriate for any notion of entailment with regard to typicality. Our first important result is a negative one, though. It is an impossibility result proving that the set of postulates cannot all be satisfied simultaneously. A more detailed analysis of the result shows that, instead of being viewed as negative, this result should rather be interpreted as an indication that PTL allows for different types of entailment, corresponding to different subsets of the full set of postulates we provide. In line with this argument, we define three types of entailment for PTL corresponding to distinct subsets of the postulates, referred to as \emph{LM-entailment},  \emph{PT-entailment}, and \emph{PT'-entailment}, a modification of the latter.  Our argument for more than one type of entailment for the same logic is in line with the proposal put forward by Lehmann in the context of entailment for conditional knowledge bases, where he proposes both \emph{prototypical reasoning} and \emph{presumptive reasoning} as acceptable forms of entailment~\cite{Lehmann1995}. We elaborate on this point in Section \ref{Weakening}, but the gist of the argument is the acknowledgement of the existence of more than one form of entailment for the same representational formalism. 
\myskip

The remainder of the present paper is structured as follows. Section~\ref{Preliminaries} provides the background and notation for the rest of the work. In Section~\ref{entailsection} we discuss the complexities surrounding a notion of entailment for PTL. In Section~\ref{TowardsEntailment} we put forward our postulates and show the impossibility result. In Section~\ref{Minimum} we define LM-entailment while Section~\ref{PT-entailment} is devoted to the definition of PT-entailment, and Section~\ref{PT'-entailment} to the definition of PT'-entailment. Section~\ref{Weakening} addresses the implications of the impossibility result, making the case for three forms of PTL entailment. Section \ref{RelatedWork} discusses related work, while Section~\ref{Conclusion} concludes and discusses  future work.

\section{Logical preliminaries}\label{Preliminaries}

Let~$\Prp$ be a finite set of propositional {\em atoms} with at least two elements.\footnote{%
This (reasonable) assumption is needed for technical reasons.
}
We use $p,q,\ldots$ as meta-variables for atoms. Propositional sentences (and, in later sections, sentences of the richer language we shall introduce in Section~\ref{BackgroundPTL} below) are denoted by $\alpha,\beta,\ldots$, and are recursively defined in the usual way: $\alpha\ \grammar p \mid \lnot\alpha \mid \alpha\land\alpha \mid \top \mid \bot$. All the other Boolean connectives ($\lor$, $\limp$, $\liff$, \ldots) are defined in terms of~$\lnot$ and~$\land$ in the standard way. With~$\Lang$ we denote the set of all propositional sentences.

We denote by $\U$ the set of all propositional {\em valuations} $v : \Prp\longrightarrow\{0,1\}$, \ie, $\U\defined\{0,1\}^{\Prp}$. Whenever it eases the presentation, we shall represent valuations as sets of literals (\ie, atoms or negated atoms), with each literal indicating the truth-value of the respective atom. Thus, for the logic generated from $\Prp=\{p, q\}$, the valuation in which~$p$ is true and~$q$ is false will be represented as $\{p, \neg q\}$.
Satisfaction of a sentence $\alpha\in\Lang$ by $v\in\U$ is defined in the usual truth-functional way and is denoted by $v\sat\alpha$.

\subsection{KLM-style rational conditionals}\label{KLM}

In the conditional logic investigated by Kraus~\etal.~\cite{KrausEtAl1990}, often referred to as the~{\em KLM approach}, one is interested in (defeasible) conditionals of the form $\alpha\twiddle\beta$, read as ``typically, if~$\alpha$, then $\beta$'' (or, depending on the example at hand, as ``$\alpha$s are typically $\beta$s'' and variants thereof). For instance, if $\Prp=\{\bird,\fly,\p\}$, where~$\bird$, $\fly$ and $\p$ stand for, respectively, ``being a bird'', ``being able to fly'', and ``being a penguin'', the following are examples of defeasible conditionals: $\bird\twiddle\fly$ (birds typically fly), $\p\land\bird\twiddle\lnot\fly$ (penguins that are birds typically do not fly).

Kraus et al. put forward the following list of properties that the conditional~$\twiddle$ ought to satisfy in order to be considered as appropriate in a non-monotonic setting (these properties have been discussed at length in the non-monotonic reasoning community and we shall not do so here):
\[
\begin{array}{llllll}
(\text{Ref}) & \alpha \twiddle \alpha & (\text{\small LLE}) & {\displaystyle \frac{\entails \alpha\liff \beta,\ \alpha \twiddle \gamma}{\beta \twiddle \gamma}} & (\text{And}) & {\displaystyle \frac{\alpha \twiddle \beta, \ \alpha \twiddle \gamma}{\alpha \twiddle \beta\land\gamma}} \\[0.5cm]
(\text{Or}) & {\displaystyle \frac{\alpha \twiddle \gamma, \ \beta \twiddle \gamma}{\alpha \lor \beta \twiddle \gamma}} & (\text{\small RW}) & {\displaystyle \frac{\alpha \twiddle \beta, \ \entails \beta \limp \gamma}{\alpha \twiddle \gamma}} & (\text{\small CM}) & {\displaystyle \frac{\alpha \twiddle \beta, \ \alpha \twiddle \gamma}{\alpha\land\beta \twiddle \gamma}} 
\end{array}
\]

A conditional satisfying such properties is called a \emph{preferential conditional}. We can require~$\twiddle$ to satisfy other properties as well, one of which is rational monotonicity:
\[
\text{({\small RM})}\quad\frac{\alpha\twiddle \gamma, \ \alpha\ntwiddle\lnot\beta}{\alpha\wedge \beta\twiddle \gamma}
\]
A preferential conditional also satisfying~(RM) is called a \emph{rational conditional}.

The semantics of~KLM-style rational conditionals is given by structures called {\em ranked interpretations}~\cite{LehmannMagidor1992}: 

\begin{definition}[Ranked interpretation]\label{Def:RankedInterpretation}
A \textbf{ranked interpretation}~$\RM$ is a function from $\U$ to $\mathbb{N}\cup\{\infty\}$ satisfying the following \emph{convexity} property: for every $i\in \mathbb{N}$, if $\RM(v)=i$, then, for every $j$ such that $0\leq j<i$, there is a $v'\in\U$ for which $\RM(v')=j$.
\end{definition}
Observe that $\RM$ generates a modular order $\pref_{\RM}$ on $\U$ as follows: $u\pref_{\RM}v$ if and only if $\RM(u)<\RM(v)$ (where $i<\infty$ for every $i\in\mathbb{N}$). If there is no ambiguity, we will omit the subscript and refer to the modular order as $\prec$.\footnote{Recall that, given a set $X$, $\pref\ \subseteq X\times X$ is modular if and only if there is total order $\leq$ on a set $\Omega$ and a ranking function $\rk:X\mapsto\Omega$ \st\ for every $x,y\in X$, $x\pref y$ if and only if $\rk(x)<\rk(y)$.}

In a ranked interpretation~$\RM$ the intuition is that valuations lower down in the ordering are deemed more normal (or typical) than those higher up, with those with an infinite rank (a rank of $\infty$) being regarded as so atypical as to be impossible. 

The \emph{possible} valuations in $\RM$ are defined as follows: $\U^{\RM}\defined\{u\in\U\mid \RM(u)<\infty\}$. 
Given $\alpha\in\Lang$, we let $\states{\alpha}{\RM}\defined\{v\in\U^{\RM}\mid v\sat\alpha\}$.
Note that it may be possible that $\RM(u)=\infty$ for every $u\in\U$, and therefore that $\U^{\RM}=\emptyset$.

Given $\alpha,\beta\in\Lang$, we say~$\RM$ satisfies (is a ranked model of) the conditional~$\alpha\twiddle\beta$ (denoted $\RM\sat\alpha\twiddle\beta$) if all the $\pref$-minimal $\alpha$-valuations also satisfy $\beta$, \ie, if 
$\min_{\pref}\states{\alpha}{\RM}\subseteq\states{\beta}{\RM}$. We say~$\RM$ is a ranked model of a set of conditionals~$\mathcal{C}$ if $\RM\sat\alpha\twiddle\beta$ for every $\alpha\twiddle\beta\in\mathcal{C}$, and that a set of conditionals~$\mathcal{C}$ is {\em satisfiable} only if it has a ranked model $\RM$ for which $\U^{\RM}\neq\emptyset$. Observe that if~$\mathcal{C}$ is unsatisfiable, it has as its only ranked model the ranked interpretation $\RM$ for which $\U^{\RM}=\emptyset$.  

Sometimes it is convenient to represent a ranked interpretation $\RM$ as a partition $(L_0, \ldots, L_{n-1},L_{\infty})$ of $\U$ where, for 
$i\in\mathbb{N}\cup\{\infty\}$, $L_i = \{u\in\U\mid \RM(u)=i\}$ and where $n$ is some $i\in\mathbb{N}$ for
which $L_i=\emptyset$. That is, for each $i\in\{0,\ldots,n-1,\infty\}$, $L_i$ is the set of all valuations of
rank $i$. We refer to such a ranked interpretation as an $n$-\emph{rank} interpretation. 

Observe that the partition above has a finite number of cells, but includes the possibility for some of the~$L_i$s to be empty. This is necessary for three reasons. First, the cell~$L_{\infty}$ (the set of all impossible valuations) may be empty. Second, it may be the case that~$L_{\infty}=\U$. That is, it may be that all valuations are impossible. Third, as we shall see below, this representation will often be used to \emph{compare} ranked interpretations. In cases where such ranked interpretations do not have the same number of non-empty cells, this representation allows us to represent them as having the same (finite) number of cells,
say $(L_0, \ldots, L_{n-1},L_{\infty})$ and $(M_0, \ldots, M_{n-1},M_{\infty})$, where $n$ is the smallest integer such that $L_i=M_i=\emptyset$. 

Figure~\ref{Figure:RankedInterpretation} depicts an example of a ranked interpretation for $\Prp=\{\bird,\fly,\p \}$ satisfying both $\bird\twiddle\fly$ and $\p\land\bird\twiddle\lnot\fly$. (In our graphical representations of the ranked interpretations we frequently omit the rank~$\infty$.)

\begin{figure}[ht]
\begin{center}
\begin{tabular}{|c|c|} 
\hline
 $2$ & 
 $\{\bird,\fly,\p\}$
 \\ \hline
 $1$ & 
 $\{\bird,\lnot\fly,\lnot\p\}$, \quad
 $\{\bird,\lnot\fly,\p\}$
 \\ \hline
 $0$ & 
 $\{\lnot\bird,\lnot\fly,\lnot\p\}$, \quad
 $\{\lnot\bird,\fly,\lnot\p\}$, \quad
 $\{\bird,\fly,\lnot\p\}$ 
 \\ \hline
\end{tabular}
\end{center}
\caption{A ranked interpretation for~$\Prp=\{\bird,\fly,\p\}$.}
\label{Figure:RankedInterpretation}
\end{figure}

For a better understanding of the reasons behind the aforementioned properties and the semantic constructions, the reader is referred to the work of Kraus~\etal.~\cite{KrausEtAl1990,LehmannMagidor1992}.

\subsection{Rational closure}\label{RationalClosure}

Given a set of conditionals~$\mathcal{C}$, reasoning in the~KLM framework amounts to the derivation of new conditionals from~$\mathcal{C}$. Towards this end, Lehmann and Magidor~\cite{LehmannMagidor1992} proposed what they refer to as {\em rational closure}. Here we focus on the semantic version of rational closure they present.

Their idea was to define a preference relation~$\mpref{\LM}$ over the set of possible ranked interpretations and then to base entailment on choosing only the most preferred, \ie, minimal \wrt~$\mpref{\LM}$, ranked models of~$\mathcal{C}$. 

The relation $\mpref{\LM}$ can be described as follows.

\begin{definition}[LM-preference]\label{Def:rel_LM}

Let $\RM_1 = (L_0, \ldots, L_{n-1},L_{\infty})$ and $\RM_2 = (M_0, \ldots, M_{n-1},M_{\infty})$ be any pair of ranked interpretations. 
Then, 
\vspace{-0.1cm}
\[
\begin{array}{ccll}
\RM_{1} \mpref{\LM} \RM_{2} & \textrm{ if } & \textit{either} &
L_{i} = M_{i}\ \textrm{for all } i\in\{0,\ldots,n-1,\infty\}, \\
& & \textit{or} &
L_{j} \supseteq M_{j} \textrm{ for the smallest $j\geq 0$ s.t.\ } L_{j} \neq M_{j}. 
\end{array}
\]
$\RM_1\mprefstrict{\LM} \RM_2$ if and only if $\RM_1\mpref{\LM} \RM_2$ and not $\RM_2\protect{\mpref{\LM}} \RM_1$.
\end{definition}

$\mpref{\LM}$ forms a partial order over ranked interpretations, 
and, for every satisfiable set of conditionals~$\mathcal{C}$, there exists a unique $\mpref{\LM}$-minimum element $\RM^{\mathrm{rc}}(\mathcal{C})$ among all the ranked models of~$\mathcal{C}$   (see Proposition~\ref{Prop:uniqueLM} in Appendix~\ref{App:Proof_Prel}). We will refer to this element as the {\em LM-minimum}.

This is not exactly the semantic representation defined by Lehmann and Magidor, but this representation can easily be derived from other work on rational closure, such as that of Booth and Paris~\cite{BoothParis1998} 
(see Appendix \ref{App:Proof_Prel}).

\begin{restatable}{proposition}{LMisRC}\label{Prop:LMisRC}
Given a set of conditionals~$\mathcal{C}$ and a conditional $\alpha\twiddle\beta$. $\alpha\twiddle\beta$ is in the rational closure of~$\mathcal{C}$ iff:

\begin{enumerate}
    \item $\mathcal{C}$ is unsatisfiable; or
    \item $\RM^{\mathrm{rc}}(\mathcal{C})\sat \alpha\twiddle\beta$.
\end{enumerate}

\end{restatable}

The idea is that those ranked interpretations should be preferred in which as many valuations as possible are judged to be as plausible as the background knowledge~$\mathcal{C}$ allows. Observe also that one of the consequences of this ordering is that, all other things being equal, a ranked interpretation in which a valuation is deemed to be possible will be preferred over one in which the same valuation is seen as impossible.

 Then the rational closure of~$\mathcal{C}$ is the set $\twiddle^\mathrm{rc}_\mathcal{C}\defined\{(\alpha, \beta) \mid\RM^{\mathrm{rc}}(\mathcal{\mathcal{C}})\sat\alpha\twiddle\beta\}$. Rational closure is commonly viewed as the \emph{basic} (although certainly not the only acceptable) form of entailment over propositional conditional knowledge bases, on which other, more venturous, forms of entailment can be constructed. It is therefore an appropriate choice on which to base our investigations into versions of entailment for~PTL. 

\subsection{Propositional Typicality Logic}\label{BackgroundPTL}

PTL~\cite{BoothEtAl2012} is a logical formalism explicitly allowing for the representation of and reasoning about a notion of {\em typicality}. Syntactically, it extends classical propositional logic with a {\em typicality operator}~$\typ$, the intuition of which is to capture the most typical (alias normal or conventional) situations or worlds. Here we shall briefly present the main results about PTL relevant for our purposes.

The language of PTL, denoted by~$\typup{\Lang}$, is recursively defined by:
\[
\alpha\ \grammar p \mid \lnot\alpha \mid \alpha\land\alpha \mid \top \mid \bot \mid \typ{\alpha}
\]
As before, $p$ denotes an atom and all the other Boolean connectives are defined in terms of $\lnot$ and $\land$.

Let~$\Prp=\{\bird,\fly,\ost,\p\}$, where $\bird$, $\fly$ and $\p$ are as before and~$\ost$ represents ``being an ostrich''. The following are examples of $\typup{\Lang}$-sentences: $\typ{\bird}$ (being a typical bird), $\ost\!\limp\!\lnot\!\typ\!\bird$ (ostriches are not typical birds), $(\p\lor\ost)\liff(\bird\land\typ\lnot\fly)$ (being a penguin or an ostrich is equivalent to being a bird and being a typical non-flying creature).

Intuitively, a sentence of the form $\typ{\alpha}$ is understood to refer to the typical situations in which~$\alpha$ holds. Note that~$\alpha$ can itself be a $\typ$-sentence. The semantics of~PTL is also in terms of ranked interpretations (see Definition~\ref{Def:RankedInterpretation}). Satisfaction is defined inductively in the classical way, adding the following condition: $v\sat\typ\alpha$ if $v\sat\alpha$ and there is no~$v'$ such that $v'\pref v$ and $v'\sat\alpha$. That is, given~$\RM$,  $\states{\typ\alpha}{\RM}\defined\min_{\pref}\states{\alpha}{\RM}$. In the ranked interpretation~$\RM$ of Figure~\ref{Figure:RankedInterpretation}, we have $\states{\typ\bird}{\RM}=\{\{\bird, \fly,\neg \p\}\}$, $\states{\typ\p}{\RM}=\{\{\bird,\neg \fly, \p\}\}$ and $\states{\typ(\bird\land\lnot\fly)}{\RM}=\{\{ \bird,\neg \fly, \neg \p\}, \{\bird,\neg \fly, \p\}\}$.

We say that $\alpha\in\typup{\Lang}$ is {\em satisfiable} in a ranked interpretation~$\RM$ if $\states{\alpha}{\RM}\neq\emptyset$, otherwise~$\alpha$ is {\em unsatisfiable} in~$\RM$. We say that~$\RM$ is a {\em ranked model} of~$\alpha$ (denoted $\RM\sat\alpha$) if $\states{\alpha}{\RM}=\U^{\RM}$. Observe that when $\U^{\RM}=\emptyset$, then $\RM$ is a model of every $\alpha\in\typup{\Lang}$.
 
For $X\subseteq\typup{\Lang}$ we define $\Mod{X}\defined\{\RM\mid\RM\sat\alpha$ for every $\alpha\in X\}$.  $X$ is {\em satisfiable} iff $X$ has at least one model $\RM$ for which $\U^{\RM}\neq\emptyset$. Observe that if $X$ is unsatisfiable, it has as its only ranked model the ranked interpretation $\RM$ for which $\U^{\RM}=\emptyset$.
A PTL {\em knowledge base} is a set of sentences~$\KB\subseteq\typup{\Lang}$.

A useful property of the typicality operator~$\typ$ is that it allows us to express KLM-style conditionals. That is, for every ranked interpretation $\RM$ and every $\alpha,\beta\in\Lang$, $\RM\sat\alpha\twiddle\beta$ if and only if $\RM\sat\typ{\alpha}\limp\beta$. 
The converse does not hold since it can be shown that there are $\typup{\Lang}$-sentences that cannot be expressed as a set of KLM-style $\twiddle$-statements on~\Lang. To give an example (taken from Booth~\etal.~\cite{BoothEtAl2013}), assuming $\Prp = \{p,q\}$ then $\typ p$ is one such sentence, since $\typ p$ has exactly four ranked models, corresponding to the cases in which $\U^{\RM}$ is respectively taken to be (1) $\{\{p,q\}, \{p, \neg q\}\}$, (2) $\{\{p,q\}\}$, (3) $\{\{p,\neg q\}\}$ and (4)~$\emptyset$ (and where, in each case the ordering $\prec_\RM$ is taken to be empty). Yet there exists no set $X$ of KLM-style $\twiddle$-statements with exactly these models.
\myskip

The representation result below, extending Theorem~3.12 of Lehmann and Magidor~\cite{LehmannMagidor1992} to $\typup{\Lang}$, shows that the formalisation of the KLM rational conditional~$\twiddle$ inside~PTL is appropriate.

\begin{observation}[Booth~\etal.~\cite{BoothEtAl2013}, Corollary~22]\label{Theorem:RepResultTwiddleFlatBar}
Let $\RM$ be a ranked interpretation and let $\twiddle_{\RM}\ \defined\{(\alpha,\beta) \mid \alpha,\beta\in\typup{\Lang}$ and $\RM\sat\typ\alpha\rightarrow\beta\}$. Then $\twiddle_{\RM}$ is a rational conditional. Conversely, for every rational conditional $\twiddle$, there exists a ranked interpretation $\RM$ such that, for every $\alpha, \beta\in\typup{\Lang}$, $\alpha\twiddle\beta$ if and only if $\RM\sat\typ\alpha\rightarrow\beta$.
\end{observation}
\myskip

For more details on PTL and the aforementioned properties, the reader is referred to the work by Booth~\etal.~\cite{BoothEtAl2013}.

\section{The entailment problem for PTL}\label{entailsection}

The purpose of this section is to provide a more formal motivation for the remainder of the paper. From the perspective of knowledge representation and reasoning (KR\&R), a central issue is that of what it means for a~PTL sentence to {\em follow} from a PTL knowledge base~\KB. An obvious approach to the matter is to embrace the notion of entailment advocated by Tarski~\cite{Tarski1930} and largely adopted in the logic-based~KR\&R community.

\begin{definition}[Ranked entailment and consequence]\label{Def:RankedEntailment}
Let $\KB$ be a PTL knowledge base and $\alpha\in\typup{\Lang}$. We say $\KB$ \textbf{ranked-entails}~$\alpha$ (noted~$\KB\NMentails_0 \alpha$) if $\Mod{\KB}\subseteq\Mod{\alpha}$.
Its associated \textbf{ranked consequence} operator is defined by setting $\Cn{0}{\KB}\defined$ $\{\alpha\in\typup{\Lang} \mid\KB\NMentails_0\alpha\}$.
\end{definition}

As we shall see below, this version of entailment is {\em not} appropriate in the context of~PTL for a number of reasons. For one, consider the following definition of a conditional induced from a set of PTL sentences. 

\begin{definition}[Induced conditional relation]\label{Def:InducedTwiddle}
Let $\KB\subseteq\typup{\Lang}$. We define  $\twiddle_{\KB}\ \defined\{(\alpha,\beta) \mid \alpha, \beta\in\Lang$ and $\typ\alpha\limp\beta\in\KB\}$.
\end{definition}

It is worth investigating whether~$\twiddle_{\Cn{0}{\KB}}$  is rational for a PTL knowledge base $\KB$, \ie, whether it satisfies all the~KLM properties for rationality from Section~\ref{KLM}. The following proposition, which mimics a similar result by~Lehmann and Magidor in the propositional case, shows that this is not the case:

\begin{observation}[Booth~\etal.~\cite{BoothEtAl2013}, Proposition~25]
For a PTL knowledge base $\KB$, $\twiddle_{\Cn{0}{\KB}}$ is a preferential conditional, but is not necessarily a rational conditional.
\end{observation}

Hence, ranked consequence as defined above delivers an induced defeasible conditional that is preferential but that need not be rational. This forms an argument against ranked entailment being an appropriate notion of entailment for~PTL.

One of the principles to give serious consideration when investigating PTL entailment is the \emph{presumption of typicality}~\cite[p.~63]{Lehmann1995}. Informally, this means that one should assume that every situation is as typical as possible. Sections~\ref{TowardsEntailment} and~\ref{PT-entailment} contain a formalisation of this principle. For now, we illustrate it with an example.

\begin{example}\label{Example:Ampliativeness}
Let $\KB_{1}=\{\p\limp\bird, \typ{\bird}\limp\fly\}$ (penguins are birds,  and typical birds  fly). Given just this information about birds and penguins, it is reasonable to expect  $\typ{\p}\limp\typ\bird$ (typical penguins  are typical birds), and therefore $\typ{\p}\limp\fly$ (typical penguins fly), to follow from~$\KB_{1}$. With ranked entailment, these requirements are not met, as there is a ranked model of~$\KB_{1}$, depicted in Figure~\ref{Figure:RankedCounterModel}, invalidating the expected conclusions. This is so because ranked entailment, being a Tarskian relation, is not ampliative, \ie, it does not allow for venturing beyond what necessarily follows from the knowledge base. \hfill \qed
\end{example}

\begin{figure}[ht]
\begin{center}
\begin{tabular}{|c|c|} 
\hline
 $2$ & 
 $\{\bird,\fly,\p\}$
 \\ \hline
 $1$ & 
 $\{\bird,\lnot\fly,\lnot\p\}$, \quad
 $\{\bird,\lnot\fly,\p\}$
 \\ \hline
 $0$ & 
 $\{\lnot\bird,\lnot\fly,\lnot\p\}$, \quad
 $\{\lnot\bird,\fly,\lnot\p\}$, \quad
 $\{\bird,\fly,\lnot\p\}$ 
 \\ \hline
\end{tabular}
\end{center}
\caption{A ranked model of~$\KB_{1}=\{\p\limp\bird, \typ{\bird}\limp\fly\}$ satisfying neither $\typ{\p}\limp\typ \bird$ nor $\typ{\p}\limp\fly$.}
\label{Figure:RankedCounterModel}
\end{figure}

Besides requiring PTL entailment to be ampliative, we also want it to be \emph{defeasible}, that is, the conclusions derived under the presumption of typicality in an ampliative way can be retracted in case of new conflicting information. This is illustrated by the following example.

\begin{example}\label{Example:Defeasibility}
Assume $\typ{\p}\limp\typ \bird$ and $\typ{\p}\limp\fly$ (somehow) could follow from~$\KB_{1}$ in Example~\ref{Example:Ampliativeness}, but then we are informed that 
typical penguins  do not fly. 
That is, let $\KB_{2}=\KB_{1}\cup\{\typ{\p}\limp\lnot\fly\}$. 
While we want $\p\limp\lnot{\typ}\bird$ (penguins are not typical birds) to follow from~$\KB_{2}$, we do {\em not} want $\typ{\p}\limp\fly$ to follow from~$\KB_{2}$, which is not possible with ranked entailment. \hfill \qed
\end{example}

\section{Towards a notion of entailment for PTL}\label{TowardsEntailment}

We have seen that ranked entailment has some serious drawbacks in a non-monotonic context. Therefore, the question as to what logical consequence in~PTL should mean remains mostly unanswered so far. In this section, we first specify and discuss a list of postulates formalising the requirements motivated in the last section and that, at first glance, seem reasonable for an appropriate notion of entailment in~PTL. In the subsequent section, we consider specific alternatives to ranked entailment and check them against our postulates. 

We start by introducing some notation. With $\NMentails_{?}\ \subseteq \ps{\typup{\Lang}}\times\typup{\Lang}$, we denote any entailment relation on the language of~PTL. Given an entailment relation~$\NMentails_{?}$, its associated {\em consequence} operator is defined in the usual way by setting, for each $\KB\subseteq\typup{\Lang}$, $\Cn{?}{\KB} \defined \{ \alpha\in\typup{\Lang} \mid \KB\NMentails_{?}\alpha \}$. 









Following the tradition in the non-monotonic reasoning literature, the obvious starting point is to consider some of the basic properties of classical consequence operators.

\begin{description}
\item[P1] For every $\KB \subseteq \typup{\Lang} $, $\KB\subseteq\Cn{?}{\KB}$ \hfill(Inclusion)
\item[P2] For every $\KB, \KB' \subseteq \typup{\Lang} $,\\ if $\KB\subseteq \KB'\subseteq \Cn{?}{\KB}$, then  $\Cn{?}{\KB'}=\Cn{?}{\KB}$ \hfill(Cumulativity)
\end{description}

Note that Cumulativity and Inclusion imply Idempotence. Idempotence, formalised as  

\vspace{0.2cm}

For every $\KB \subseteq \typup{\Lang}$, $\Cn{?}{\KB}=\Cn{?}{\Cn{?}{\KB}}$ \hfill(Idempotence)
    
\vspace{0.2cm}

\noindent can be derived from Cumulativity by setting $\KB'=\Cn{?}{\KB}$, and letting Inclusion impose the satisfaction of the antecedent. Idempotence indicates that a consequence operator behaves as a `once-off' operation, that is,  as a closure operator. There is agreement in the literature that both Inclusion and Cumulativity  are desirable properties to have \cite[p.43]{Makinson1994}.




Ranked entailment, as defined in Section~\ref{entailsection}, satisfies Properties~{\bf P1} and~{\bf P2}.  Nevertheless, $\Cn{0}{\cdot}$, the associated consequence relation of ranked entailment, also satisfies the classical property of Monotonicity: If $\KB\subseteq\KB'$, then $\Cn{0}{\KB}\subseteq\Cn{0}{\KB'}$. As seen in Example~\ref{Example:Defeasibility}, this is a property that we do not want $\Cn{?}{\cdot}$ to satisfy (certainly not in general).

So, we require~$\Cn{?}{\cdot}$ to be a {\em non-monotonic} consequence operator. This amounts to requiring~$\Cn{?}{\cdot}$ to satisfy the following two postulates:
\begin{description}
\item[P3] For every $\KB\subseteq\typup{\Lang}$,  $\Cn{0}{\KB}\subseteq\Cn{?}{\KB}$ \hfill(Ampliativeness)
\item[P4] For some $\KB,\KB'\subseteq\typup{\Lang}$, $\KB\subseteq\KB'$ but $\Cn{?}{\KB}\not\subseteq\Cn{?}{\KB'}$ \hfill(Defeasibility)
\end{description}

Ampliativeness, a property generalising supra-classicality~\cite{Makinson2005a} (where the basic underlying entailment relation is classical), says that $\Cn{?}{\cdot}$ should be at least as venturous as its underlying ranked entailment. Defeasibility specifies that~$\Cn{?}{\cdot}$ should be flexible enough to disallow previously derived conclusions in the light of new (possibly conflicting) information. In Example~\ref{Example:Ampliativeness}, assuming $\typ\p\limp\fly\in\Cn{?}{\KB_{1}}$ is the case, then $\typ\p\limp\fly$ should no longer be concluded if~$\typ\p\limp\lnot\fly$ is added to~$\KB_{1}$. Note that adding Defeasibility to Ampliativeness actually implies a {\em strict} version of Ampliativeness which says $\Cn{?}{\cdot}$ should in some cases be {\em more} venturous than its underlying ranked entailment. (Since, if $\Cn{?}{\KB} = \Cn{0}{\KB}$ for all $\KB$, then $\Cn{?}{\cdot}$ is just ranked entailment, which is monotonic.)




{\bf P1}, {\bf P2} and {\bf P3} together imply that the closure operation  $\Cn{?}{\cdot}$ gives as output a theory that is closed under $\Cn{0}{\cdot}$.

\begin{lemma}\label{CN0_closure}
If $\Cn{?}{\cdot}$ satisfies {\bf P1}, {\bf P2} and {\bf P3}, then, for every $\KB \subseteq \typup{\Lang}$,
\[\Cn{?}{\KB}=\Cn{0}{\Cn{?}{\KB}}\]

\end{lemma}

\begin{proof}
$\Cn{0}{\cdot}$ is a Tarskian consequence relation (see Definition \ref{Def:RankedEntailment}), and, as such, it satisfies \emph{Inclusion}. That is, for every set of formulas $\mathcal{S}$, $\mathcal{S}\subseteq\Cn{0}{\mathcal{S}}$. To see it, it is sufficient to check that, according to Definition \ref{Def:RankedEntailment}, for every $\alpha \in\mathcal{S}$, $\mathcal{S}\NMentails_0\alpha$. Hence, since $\Cn{0}{\cdot}$ satisfies Inclusion, $\Cn{?}{\KB}\subseteq\Cn{0}{\Cn{?}{\KB}}$.

By P3 we have $\Cn{0}{\Cn{?}{\KB}}\subseteq \Cn{?}{\Cn{?}{\KB}}$, that, by Idempotence (consequence of {\bf P1} and {\bf P2}), implies $\Cn{0}{\Cn{?}{\KB}}\subseteq \Cn{?}{\KB}$.
\end{proof}


Similarly to KLM in the propositional case, we would ideally like the defeasible conditional induced by~$\Cn{?}{\KB}$ (see Definition~\ref{Def:InducedTwiddle}) to satisfy all the rationality properties:
\begin{description}
\item[P5] For every $\KB\subseteq\typup{\Lang}$, $\twiddle_{\Cn{?}{\KB}}$ is a rational conditional relation on~\Lang \hfill(Conditional Rationality)
\end{description}

As observed above, {\bf P5} requires the defeasible conditional induced by~$\Cn{?}{\KB}$ to be rational---that is, to satisfy all the rationality properties. But from Theorem 3.12 of Lehmann and Magidor~\cite{LehmannMagidor1992} it follows that every rational defeasible conditional can be obtained from a single ranked interpretation. So, from this it follows that requiring the defeasible conditional induced by~$\Cn{?}{\KB}$ to be rational 
amounts to requiring that the defeasible conditional be generated by a single ranked interpretation. That is, by courtesy of this result, {\bf P5} can also be rephrased as follows:
\begin{description}
\item[P5'] For every $\KB\subseteq\typup{\Lang}$, there is a ranked interpretation~$\RM$ \st, for every $\alpha,\beta\in\Lang$, $\alpha\twiddle_{\Cn{?}{\KB}}\beta$ if and only if $\RM\sat\typ{\alpha}\rightarrow\beta$.
 \hfill($\twiddle$ Single Model)
\end{description}

The next postulate we consider, which is easily shown to be a strengthening of {\bf P5}, simply applies this same requirement, not just to defeasible statements, but to all statements expressible in PTL:
\begin{description}
\item[P6] For every $\KB\subseteq\typup{\Lang}$, there is a ranked interpretation~$\RM$ \st, for all $\alpha\in\typup{\Lang}$, $\alpha\in\Cn{?}{\KB}$ if and only if $\RM\sat\alpha$
 \hfill(Single Model)
\end{description}

An important special case of a PTL knowledge base is when the individual elements of $\KB$ correspond to KLM-style conditionals.

\begin{definition}
[(Propositional) conditional knowledge base]
A PTL knowledge base $\KB$ will be called a (propositional) \textbf{conditional knowledge base} if each element of $\KB$ is of the form $\typ\alpha \limp \beta$, for $\alpha, \beta \in \Lang$.
\end{definition}

The next postulate says that if~$\KB$ is a propositional conditional knowledge base, then the result should coincide with Lehmann and Magidor's definition of rational closure:
\begin{description}
\item[P7] For every conditional knowledge base $\KB$, $\twiddle_{\Cn{?}{\KB}}=\ \twiddle^\mathrm{rc}_\KB$  \hfill(Respects Rational Closure)
\end{description}

{\bf P7} implies {\bf P4}, since rational closure is a non-monotonic closure operation.


The following property was shown by Lehmann and Magidor to be satisfied by the rational closure for \emph{conditional} knowledge bases. 
\begin{description}
\item[P8] For every $\KB\subseteq\typup{\Lang}$ and $\alpha\in\Lang$, $\alpha\in\Cn{?}{\KB}$ if and only if $\alpha\in\Cn{0}{\KB}$  \hfill(Strict Entailment)
\end{description}
{\bf P8} states that~$\Cn{?}{\cdot}$ should coincide with ranked entailment for those sentences not involving typicality. The motivation for Strict Entailment is twofold. First, it is a proposal for ranked entailment to be the lower bound for entailment \wrt\ classical sentences (those not involving typicality), a proposal that is not controversial. But secondly, it also requires entailment of classical sentences to correspond to exactly those sanctioned by ranked entailment. This can be viewed as adhering to the principle of \emph{minimal change}. Being Tarskian, ranked entailment is monotonic, and the argument is therefore that, while non-monotonicity may be applicable for sentences involving typicality, it should not be applicable to classical statements.

We are also interested in a couple of progressively weaker versions of Strict Entailment (the reasons for that will become clear later on). The first restricts it to hold only  when $\KB$ is a conditional knowledge base.
\begin{description}
\item[P9] For every conditional knowledge base $\KB$ and $\alpha\in\Lang$, $\alpha\in\Cn{?}{\KB}$ if and only if $\alpha\in\Cn{0}{\KB}$  \hfill{(Conditional Strict Entailment)}
\end{description}
Note that {\bf P7} also implies~{\bf P9}. To see this, first it is easy to check that every propositional formula $\alpha$ is equivalent to the PTL formula $\typ\neg\alpha\limp\bot$.

\begin{proposition}\label{prop:strict_cond}
For every formula $\alpha\in\Lang$ and every ranked interpretation $\RM$, $\RM\sat\alpha$ iff $\RM\sat\typ\neg\alpha\limp\bot$.
\end{proposition}

\begin{proof}
$\RM\sat\alpha$ implies $\states{\alpha}{\RM}=\U^{\RM}$, that is equivalent to $\RM\sat\neg\alpha\limp\bot$, that, in turn, implies $\RM\sat\typ\neg\alpha\limp\bot$. In the opposite direction, $\RM\sat\typ\neg\alpha\limp\bot$ means that for every $u\in\U^{\RM}$, $u\not\sat \typ \neg \alpha$. $u\not\sat \typ \neg \alpha$ for every $u\in\U^{\RM}$ implies that for every $u\in\U^{\RM}$, $u\not\sat \neg \alpha$: if we had a valuation $v$ satisfying $\neg \alpha$ in some cell $L_i$, with $i<\infty$, we would either have that  $v\sat\typ\neg\alpha$, or there would be a valuation $v'$ in some $L_j$, $j<i$, such that $v'\sat\typ\neg\alpha$. Consequently, $u\sat \alpha$ for every $u\in\U^{\RM}$, that is, $\RM\sat\alpha$. 
\end{proof}

{\bf P7} implies that, for every $\alpha\in\Lang$ and every conditional knowledge base $\KB$,  $(\alpha,\bot)\in \twiddle_{\Cn{?}{\KB}}$ iff $(\alpha,\bot)\in \twiddle^\mathrm{rc}_\KB$. 
A well-known result by Lehmann and Magidor \cite[Lemma 5.16]{LehmannMagidor1992} states that for every $\alpha\in\Lang$ and every conditional knowledge base $\KB$, $\alpha\twiddle\bot$ is in the rational closure of $\KB$ iff $\alpha\twiddle\bot$ is a  ranked consequence of $\KB$, that is, $(\alpha,\bot)\in \twiddle^\mathrm{rc}_\KB$ iff $(\alpha,\bot)\in \twiddle_{\Cn{0}{\KB}}$.
Hence we have that for every $\alpha\in\Lang$ and every conditional knowledge base $\KB$,  $(\alpha,\bot)\in \twiddle_{\Cn{?}{\KB}}$ iff $(\alpha,\bot)\in \twiddle_{\Cn{0}{\KB}}$,
that, together with Proposition \ref{prop:strict_cond}, implies {\bf P9}.

In turn, {\bf P9} implies that entailment for PTL coincides with classical propositional entailment in the case of propositional knowledge bases, as formalised by the next property. 
\begin{description}
\item[P9'] For every $\KB \subseteq \Lang$ and $\alpha\in\Lang$, $\alpha\in\Cn{?}{\KB}$ if and only if $\KB$ entails $\alpha$ in classical propositional logic.  \hfill{(Classical Entailment)}
\end{description}

Since for every $\KB \cup \{\alpha\}\subseteq\Lang$, $\KB$ entails $\alpha$ in classical propositional logic if and only if $\alpha\in\Cn{0}{\KB}$, and any $\alpha \in \Lang$ is equivalent $\typ\neg\alpha \limp \bot$, {\bf P9'} is indeed a weakening of {\bf P9} (provided that {\bf P8} also holds). 

Finally, we consider another property shown by Lehmann and Magidor to be satisfied by the rational closure for conditional knowledge bases.
\begin{description}
\item[P10] For every $\KB\subseteq\typup{\Lang}$ and $\alpha\in\Lang$, $\typ\top\limp\alpha\in\Cn{?}{\KB}$ if and only if $\typ\top\limp\alpha\in\Cn{0}{\KB}$  \hfill(Typical Entailment)
\end{description}

The motivation for {\bf P10} is similar to that for {\bf P8} above in that  we want to constrain what should hold in the most typical situations. That is, given a knowledge base, the property speaks to which formulas of the form $\typ\top\rightarrow\alpha$ should follow. Ranked entailment clearly provides  a lower bound for such a kind of statement, but {\bf P10} also proposes to consider ranked entailment as the upper bound, thereby requiring that the  set of statements $\typ\top\rightarrow\alpha$ entailed by a knowledge base should  correspond exactly to those sanctioned by ranked entailment. The argument for this is that ranked entailment is monotonic and, applying the principle of minimal change, it is only when dealing with atypical situations that ranked entailment is not always sufficient.


Although these postulates all seem reasonable on their own, it turns out that they cannot all be satisfied simultaneously. In fact, this impossibility result already holds for a strict subset of the postulates.






\begin{restatable}{theorem}{impossibility}\label{Impossibility}
There is no PTL consequence operator $\Cn{?}{\cdot}$ satisfying all of~{\bf P1}, {\bf P2}, {\bf P3}, {\bf P5}, {\bf P8} and~{\bf P10}.
\end{restatable}

\begin{proof}
Regarding {\bf P5}, requiring $\twiddle_{\Cn{?}{\cdot}}$ to satisfy (RM) is equivalent to requiring that, for every knowledge base $\KB$ and whatever formulas $\alpha,\beta,\gamma$, if $\typ\alpha\rightarrow \gamma\in\Cn{?}{\cdot}$ and $\typ\alpha\rightarrow \beta\notin\Cn{?}{\cdot}$, then we have $\typ(\alpha\land\neg \beta)\rightarrow \gamma\in\Cn{?}{\cdot}$.

Assume $\Cn{?}{\cdot}$ satisfies the given properties, and let $\KB = \{ \typ\top \rightarrow p, \typ\neg p\rightarrow \typ q\}$. By Strict Entailment ({\bf P8}), $p \not\in \Cn{?}{\KB}$ (because of \eg\ the 2-rank model $(\{ \{p, \neg q\} \} , \{ \{\neg p, q\} \})$ of $\KB$). By Typical Entailment ({\bf P10}), $\typ\top \rightarrow \neg q \not\in \Cn{?}{\KB}$ (because of \eg\ the 1-rank model $(\{ \{ p,q\}, \{p, \neg q\} \})$ of $\KB$). By Inclusion~({\bf P1}) $ \typ\top \rightarrow p\in \Cn{?}{\KB}$, and then by (RM) we must conclude that $\typ(\top\land q)\rightarrow p\in \Cn{?}{\KB}$, that is, $(\top\land q, p)\in \twiddle_{\Cn{?}{\KB}}$; since $\twiddle_{\Cn{?}{\cdot}}$ must satisfy LLE, the latter implies $(q, p)\in \twiddle_{\Cn{?}{\KB}}$, that is,     $\typ q\rightarrow p\in \Cn{?}{\KB}$.

Since by Inclusion~({\bf P1}) $ \typ\neg p \rightarrow \typ q\in \Cn{?}{\KB}$, we have $ \{\typ q\rightarrow p,\typ\neg p \rightarrow \typ q\}\subset \Cn{?}{\KB}$. Since $ \typ\neg p \rightarrow p\in \Cn{0}{\{\typ q\rightarrow p,\typ\neg p \rightarrow \typ q\}}$ and  $\Cn{0}{\cdot}$ is monotonic, we have $ \typ\neg p \rightarrow p\in \Cn{0}{\Cn{?}{\KB}}$. 
Then, by Lemma \ref{CN0_closure}, that assumes {\bf P1}, {\bf P2} and {\bf P3}, we have that $ \typ\neg p \rightarrow p\in \Cn{?}{\KB}$.

We have that $ p\in \Cn{0}{\{\typ\neg p \rightarrow p\}}$ holds: let $\RM\sat \typ\neg p \rightarrow p$, and let $v$ be a world in $\RM$ s.t. $v\sat \neg p$. $v$ cannot satisfy $\typ\neg p$, since we would have that $v\sat \neg p\land p$; but $v\sat \neg p$ and $v\not\sat \typ \neg p$ implies that in $\RM$ there is a world $w$, such that $w\pref v$ and $w\sat  \typ \neg p$, that, again, implies $w\sat \neg p\land p$.

From $ p\in \Cn{0}{\{\typ\neg p \rightarrow p\}}$, $\typ\neg p \rightarrow p\in \Cn{?}{\KB}$, and the monotonicity of ranked entailment, we must conclude also $ p\in \Cn{0}{\Cn{?}{\KB}}$, that is, by Lemma \ref{CN0_closure}, $p\in \Cn{?}{\KB}$, against {\bf P8}.
\end{proof}
\myskip


While, at first glance, this seems to be a negative result, our contention is that it should be interpreted as an indication that a logic as expressive as PTL admits more than one form of entailment. We elaborate directly on this point in Section~\ref{Weakening}, and indirectly in Sections~\ref{Minimum} and \ref{PT-entailment}, where we define and discuss two instances of entailment for PTL.

\section{LM-entailment}\label{Minimum}

We now come to our first construction of an entailment relation in PTL. 
We first observe that there is nothing to stop us from using the preference relation $\mpref{\LM}$ (see Section~\ref{RationalClosure}) to compare ranked interpretations of {\em any} PTL knowledge base $\KB$. The question then is, does there always exist a {\em unique} LM-minimum element of the ranked models of $\KB$, as there does in the restricted conditional case? And if so, how can we construct it? We now answer these questions. 

We assume as input a PTL knowledge base $\KB$, where each sentence $\alpha\in\KB$ is in {\em normal form}:

\begin{definition}[Normal form]\label{normalform}
$\alpha \in \typup{\Lang}$ is \textbf{in normal form} if it is of the form $\bigwedge_{i \leq t} \bullet \theta_i \rightarrow 
(\phi \vee \bigvee_{i \leq s} \bullet\psi_i)$, where $t,s \geq 0$ and the  $\theta_i$, $\phi$ and $\psi_i$  are all purely propositional sentences.
\end{definition}

\begin{restatable}{theorem}{Theorem:NormalForm}\label{Theorem:NormalForm}
The normal form is complete for $\typup{\Lang}$, \ie, for every sentence $\alpha\in\typup{\Lang}$ there is a (finite) set of sentences~$X\subseteq\typup{\Lang}$, each one in normal form, such that $\Mod{\alpha}=\Mod{\bigwedge X}$.
\end{restatable}
\begin{proof}
From the results by Booth~\etal.~\cite[Section~4]{BoothEtAl2012}, it follows that we need only consider sentences with non-nested instances of the typicality operator. So we let $\alpha$ be such a sentence. We let the set of \emph{typicality} atoms be the propositional atoms occurring in $\typup{\Lang}$ together with every sentence of the form $\typ\beta$ where $\beta$ is a propositional sentence (we refer to the latter as \emph{pure typicality atoms}). And we define the set of \emph{typicality literals} in the obvious way: the set of typicality atoms and their negations. The set of \emph{pure typicality literals} consists of the pure typicality atoms and their negations.

Now we define \emph{typicality conjunctive normal form} as a conjunctive normal form defined on typicality atoms. It follows immediately that $\alpha$ can be rewritten as a sentence, say $\alpha'$, in typicality conjunctive normal form. Let~$X'$ be the set of conjuncts occurring in $\alpha'$. We show below how to rewrite each conjunct in~$X'$ into a sentence in normal form. The resulting set~$X$ of sentences in normal form is the set referred to above. 

By construction, each sentence $\gamma\in X'$ is a disjunction of typicality literals. We separate them into three disjoint sets, the set of propositional literals, the set of positive pure typicality literals (with cardinality of, say $t$, where $t\geq 0$) and the set of negative pure typicality literals (with cardinality of, say $s$, where $s\geq 0$). Let $\phi$ be the disjunction of propositional literals, denote the $s$ positive pure typicality literals 
by $\psi_1,\ldots,\psi_s$, and the $t$ negative pure typicality literals by $\theta_1,\ldots\theta_t$. It follows immediately that $\gamma$ can be rewritten as the sentence $\bigwedge_{i\leq t}\theta_i \limp (\phi \lor\bigvee_{i\leq s}\psi_i)$.
\end{proof}
\myskip

For any ranked interpretation $\RM$, and $S \subseteq \U^{\RM}$, let $\RM^{\infty}_S$ be the ranked interpretation
such that $\RM^{\infty}_S(v)=\RM(v)$ for every $v\in S$, and $\RM^{\infty}_S(v)=\infty$ for every $v\in\U\setminus S$.
That is, $\RM^{\infty}_S$ is the ranked interpretation obtained from $\RM$ by turning all valuations not in $S$ into impossible valuations. Similarly, let $\RM^1_S$ be the ranked interpretation
such that $\RM^1_S(v)=\RM(v)$ for every $v\in S$, and $\RM^1_S(v)=\RM(v)+1$ for every $v\in\U\setminus S$.
That is, $\RM^1_S$ is the ranked interpretation obtained from $\RM$ by increasing the rank of all valuations not in $S$ by 1. 



Given a PTL knowledge base $\KB$ we now define a ranked interpretation $\RM^*_{\KB}$, obtained from $\KB$, as follows:
\begin{description}
\item[Step 1] Set $\RM_0(v)\defined 0$ for all $v\in\U$, $S_0\defined\emptyset$, and $i\defined 1$;
\item[Step 2] $S_1\defined \states{\KB}{\RM_0}$ ({\em separate the valuations which satisfy $\KB$ \wrt\ the current ranked interpretation $\RM_0$ from those that do not});
\item[Step 3] If $S_i = S_{i-1}$, then $\RM^*_{\KB}\defined(\RM_i)^{\infty}_{S_i}$, and stop. 
({\em if there is no change in the new $S_i$ then set the rank of those valuations that do not satisfy $\KB$ \wrt\ $\RM_i$ to $\infty$, let $\RM^*_{\KB}$ be the interpretation that remains}, and stop);
\item[Step 4] Otherwise $\RM_i\defined (\RM_{i-1})^1_{S_i}$
({\em otherwise create a new ranked interpretation $\RM_i$ by increasing the rank of every valuation not in $S_i$ by 1});
\item[Step 5] $S_{i+1} \defined \states{\KB}{\RM_i}$ and $i\defined i+1$ ({\em separate the valuations which satisfy $\KB$ \wrt\ the current ranked interpretation $\RM_i$ from those that do not, and increment $i$});
\item[Step 6] Go to Step 3.
\end{description} 


Algorithm~\ref{LMminimal} below gives a compact description of the steps above. Note that if the input to the algorithm, $\KB$, is unsatisfiable, the ranked interpretation $\RM^*_{\KB}$ that it returns is such that $\U^{\RM^*_{\KB}}=\emptyset$.

\begin{algorithm}[ht]
\caption{LM-minimal\label{LMminimal}}
{\small
 \KwIn{$\KB$}
 \KwOut{$\RM^{*}_{\KB}$}
 $\Prp_{\KB} \defined \{p\mid p \text{ is a  propositional letter occurring in \KB}\}$\;
 Let $\U$ be the universe of valuations for the vocabulary $\Prp_{\KB}$\;
 $\RM_{0}(v)\defined 0$ for every $v\in\U$\;
 $S_0\defined \emptyset$\;
 $S_1 \defined \states{\KB}{\RM_0}$\;
 $i\defined 1$\;
 
%
\While{$S_i \neq S_{i-1}$}{
 $\RM_i\defined(\RM_{i-1})^1_{S_i}$\;
 $S_{i+1} \defined \states{\KB}{\RM_i}$\;
 $i \defined i+1$\;  
}
$\RM^{*}_{\KB}\defined(\RM_{i-1})^{\infty}_{S_i}$ \;
\Return{$\RM^*_{\KB}$}
}
\end{algorithm}

\begin{example}\label{rceg}
Let us assume, for the sake of the example, that we are only talking about birds. Let $\KB \defined\{\typ{\top}\limp(\lnot\p\land\lnot\rob), \typ\p\limp\typ\lnot\fly,\typ\rob\limp\typ\fly, \p\limp\neg \rob\}$ (the most typical things are neither penguins nor robins, typical penguins are typical non-flying birds, and typical robins are typical flying birds, penguins are not robins). The procedure initialises with all valuations being assigned the rank of $0$. The only valuations that satisfy all three sentences \wrt\ $\RM_0$ are those satisfying both $\lnot\p$ and $\lnot\rob$. Thus $S_1 \defined \states{\KB}{\RM_0} = \{\{\lnot\fly,\lnot\p,\lnot\rob\}, \{\fly,\lnot\p, \lnot\rob\}\}$ and so we obtain $\RM_1$ by changing the rank of all valuations not in $S_1$ to $1$. Note that $\states{\typ\lnot\fly}{\RM_1} = \{ \{\lnot\fly,\lnot\p, \lnot\rob\}\}$ and $\states{\typ\fly}{\RM_1} = \{ \{\fly,\lnot\p,\lnot\rob\}\}$, so we can see that none of the valuations in $\U\setminus S_1$ is able to satisfy either $\typ\p\limp\typ\lnot\fly$ or $\typ\rob\limp\typ\fly$ \wrt\ $\RM_1$. As a consequence, $S_2 \defined \states{\KB}{\RM_1} = S_1$ and so the procedure terminates here with $\RM^*_{\KB}$ as the ranked interpretation in which all valuations in $S_1$ ($\{\lnot\fly,\lnot\p,\lnot\rob\}$ and $\{\fly,\lnot\p,\lnot\rob\}$) have rank $0$ and all other valuations have rank $\infty$. 
See Figure~\ref{Fig:LM} for the ranked interpretations generated by this example. \hfill \qed
\end{example}

\begin{figure}[h]
\begin{center}
\scalebox{0.8}{
$\RM_0$~\begin{tabular}{ | c|c| } 
 \hline
 $0$ & $
 \{\lnot\fly,\lnot\p,\lnot\rob\}$, 
 $\{\lnot\fly,\lnot\p,\rob\}$, 
 $\{\lnot\fly,\p,\lnot\rob\}$, 
 $\{\lnot\fly,\p,\rob\}$, 
 $\{\fly,\lnot\p,\lnot\rob\}$,
 $\{\fly,\lnot\p,\rob\}$,
 $\{\fly,\p,\lnot\rob\}$,
 $\{\fly,\p,\rob\}$
 \\ 
\hline
\end{tabular}
}
\end{center}

\begin{center}
\scalebox{0.8}{
$\RM_1$~\begin{tabular}{ | c|c| } 
 \hline
 $1$ & 
 $\{\lnot\fly,\lnot\p,\rob\}$, 
 $\{\lnot\fly,\p,\rob\}$, 
 $\{\fly,\lnot\p,\lnot\rob\}$,
 $\{\fly,\lnot\p,\rob\}$,
 $\{\fly,\p,\lnot\rob\}$,
 $\{\fly,\p,\rob\}$
 \\ 
 \hline
 $0$ & 
 $\{\lnot\fly,\lnot\p,\lnot\rob\}$, 
 $\{\fly,\lnot\p,\lnot\rob\}$
 \\ 
 \hline
\end{tabular}
}
\end{center}

\begin{center}
\scalebox{0.8}{
$\RM^*_{\KB}$ \begin{tabular}{|c|c|} 
 \hline
 $\infty$ & 
 $\{\lnot\fly,\lnot\p,\rob\}$, 
 $\{\lnot\fly,\p,\rob\}$, 
 $\{\fly,\lnot\p,\lnot\rob\}$,
 $\{\fly,\lnot\p,\rob\}$,
 $\{\fly,\p,\lnot\rob\}$,
 $\{\fly,\p,\rob\}$
 \\ 
 \hline
 $0$ & 
 $\{\lnot\fly,\lnot\p,\lnot\rob\}$, 
 $\{\fly,\lnot\p,\lnot\rob\}$
 \\ 
 \hline
\end{tabular}
}
\end{center}

\begin{center}
\scalebox{0.8}{
$\RM^*_{\KB}$ with the valuations of rank~$\infty$ omitted:~\begin{tabular}{ | c|c| } 
 \hline
 $0$ & 
 $\{\lnot\fly,\lnot\p,\lnot\rob\}$, 
 $\{\fly,\lnot\p,\lnot\rob\}$
 \\ 
 \hline
\end{tabular}
}
\end{center}
\caption{The ranked interpretations generated in Example \ref{rceg}.}\label{Fig:LM}
\end{figure}

%

We now proceed to show that: ({\em i})~the algorithm always terminates if $\KB$ is finite; ({\em ii}) the ranked model $\RM^*_{\KB}$ it returns is a ranked model of $\KB$, and ({\em iii}) for any other ranked model~$\RM$ of $\KB$, we have $\RM^*_{\KB} \mpref{\LM} \RM$. We know the following about ({\em i})~and~({\em ii}):

\begin{restatable}{lemma}{TerminationAndSoundness}\label{Lemma:TerminationAndSoundness}
The following hold for each $i \geq 0$:
\begin{enumerate}
\item $S_{i} \subseteq S_{i+1}$, \ie, $[S_0 \subseteq S_1$ and, for all $i \geq 0$, $\states{\KB}{\RM_i} \subseteq \states{\KB}{\RM_{i+1}}]$;
\item For all $v_1, v_2 \in \U$, if $\RM_i(v_1) < \RM_i(v_2)$, then $v_1 \in \states{\KB}{\RM_i}$;
\item $\RM_i$ is a ranked interpretation.
\end{enumerate}
\end{restatable}
\begin{proof}
See Appendix~\ref{Proof:TerminationAndSoundness}.
\end{proof}
\myskip

From Item~1 in Lemma~\ref{Lemma:TerminationAndSoundness} above, we know the algorithm terminates if $\KB$ is finite, since it generates a sequence of ranked  interpretations (by Item~3) in which the set of valuations satisfying $\KB$ increases monotonically from one ranked interpretation to the next. Since each of these is finite, and since there is a finite number of valuations, the stopping criterion in Line~7 of the algorithm is guaranteed to occur eventually. 

To show that the algorithm returns a ranked model of $\KB$ it suffices to show the following.

\begin{restatable}{lemma}{IsModelOfKB}\label{Lemma:IsModelOfKB}
For every $i>0$, $(\RM_i)^{\infty}_{S_i}$ is a ranked model of $\KB$.
\end{restatable}
\begin{proof}
See Appendix~\ref{Proof:IsModelOfKB}.
\end{proof}
\myskip

So, at each stage of the algorithm, the current ranked interpretation, when those valuations not satisfying $\KB$ are excluded, forms a ranked model of $\KB$. Since the  output $\RM^*_{\KB}$ takes precisely this form we have the following result.

\begin{restatable}{proposition}{Prop:SatisfiesKB}\label{Prop:SatisfiesKB}
$\RM^*_{\KB}\sat \bigwedge \KB$.  
\end{restatable}
\begin{proof}
Follows from Lemma~\ref{Lemma:IsModelOfKB} and the construction of $\RM^*_{\KB}$.
\end{proof}
\myskip

Next we want to show that for any other ranked model~$\RM$ of~$\KB$, we have $\RM^{*}_{\KB} \mpref{\LM} \RM$. 
\begin{restatable}{lemma}{LayerInclusion}\label{Lemma:LayerInclusion}
Let $\RM^{*}_{\KB} \defined (L_0, \ldots, L_{n-1},L_{\infty})$ and let $\RM \defined (M_0, \ldots, M_{n-1},M_{\infty})$ be any other ranked model of~$\KB$. 
Let $i\in\{0,\ldots,n-1\}$. If $L_j = M_j$ for all $j<i$, then $M_i \subseteq L_i$.
\end{restatable}
\begin{proof}
See Appendix~\ref{Proof:LayerInclusion}.
\end{proof}
\myskip

From this lemma we can state: 

\begin{restatable}{proposition}{Prop:MorePreferred}\label{Prop:MorePreferred}
Consider any $\KB$ and let $\RM$ be a ranked model of $\KB$. Then $\RM^{*}_{\KB} \mpref{\LM} \RM$.
\end{restatable}

That Algorithm~\ref{LMminimal} runs in time that is (singly) exponential in the size of the input knowledge base~$\KB$ whenever $\KB$ is finite is not hard to see. Let $|\KB|=k$ and $|\Prp_{\KB}|=j$. The procedure starts by computing the universe~$\U$ of all valuations for the vocabulary~$\Prp_{\KB}$, and therefore we have~$|\U|=2^{j}$. Next, in the first round of the loop, each sentence in~$\KB$ has to be checked against all of the exponentially many valuations in~$\U$, which amounts to~$k\times2^{j}$ verifications. In the worst-case scenario, only one valuation is kept at level~0, with all the others moved up to level~1. In the next round, each sentence in~$\KB$ has to be checked against the $2^{j}-1$ valuations at level~1, but also against the only valuation at level~0, because the truth of $\typ$-sentences in a model also depends on those valuations that are lower down in the model. This amounts to~$k\times2^{j}$ verifications, which in the worst case will again result in a single valuation kept at level~1 with all the~$2^{j}-2$ ones moved up to level~2, and a number of~$k\times2^{j}$ checks to be performed in the next round. By repeating this argument one can see that, in the worst case, the algorithm will have built a ranked interpretation consisting of~$2^{j}$ layers, each one containing a single valuation, i.e., a linear ordering on the~$2^{j}$ valuations. This process will have involved~$2^{j}$ runs, each run requiring~$k\times2^{j}$ valuation checks to create a new layer. It remains to know the cost of checking whether a sentence is satisfied by a valuation in a ranked model. In the first run of the loop, namely when there is a single layer, since the preference relation at this stage is empty, each of such verifications amounts to a propositional verification, which is a polynomial-time task. From the second run of the loop onward, \ie, when truth depends on the lower layers, we have that all valuations at the lower layers have to be inspected, which in the worst case amounts to~$m\times2^{j}$ checks to be performed, with~$m$ the number of sub-formulas of the one being checked. Putting the results together, we have that in the worst case there are a maximum of $2^{j}$ runs of the main loop, each with $k\times2^{j}$ checks, and each of such valuation checks taking at most~$m'\times2^{j}$ operations, with~$m'$ the number of sub-formulas in~$\KB$, \ie, $m'=2^{\ell}$, for some~$\ell$. Hence the algorithm runs in~$2^{j}\times (k\times2^{j})\times(2^{\ell}\times2^{j})=k\times2^{3j+\ell}$, and is therefore in ExpTime.
\myskip

We are now in a position to define our first form of entailment for PTL. 
\begin{definition}[LM-entailment]\label{Def:LM-Entailment}
Let $\KB\subseteq\typup{\Lang}$ and $\alpha\in\typup{\Lang}$. We say $\KB$ \textbf{LM-entails} $\alpha$, denoted $\KB\NMentails_{\LM}\alpha$, if $\RM^{*}_{\KB}\sat\alpha$. Its corresponding consequence operator is defined as $\Cn{\LM}{\KB} \defined \{\alpha \in  \typup{\Lang} \mid \RM^{*}_{\KB} \sat \alpha\}$. 
\end{definition}

The next result outlines which properties from the previous section are satisfied by $\Cn{\LM}{\cdot}$.

\begin{restatable}{theorem}{Thm:LM-properties}{}\label{Thm:LM-properties}
$\Cn{\LM}{\cdot}$ satisfies {\bf P1}--{\bf P7}, {\bf {\bf P9}}, and {\bf P10}, but {\bf not}~{\bf P8}.
\end{restatable}
\begin{proof}
For {\bf P1}, Proposition~\ref{Prop:SatisfiesKB} guarantees that $\RM^{*}_{\KB}$ is a model of $\KB$. About~{\bf P2}, by Proposition~\ref{Prop:MorePreferred}, $\RM^{*}_{\KB}$ is the LM-minimum model of $\KB$. If $\KB\subseteq\KB'\subseteq \Cn{\LM}{\KB}$, then $\Mod{\KB'}\subseteq\Mod{\KB}$ and  $\RM^{*}_{\KB}\in \Mod{\KB'}$; consequently $\RM^{*}_{\KB}$ must also be  the LM-minimum model of $\KB'$. For {\bf P3}, note that $\RM^{*}_{\KB}$ is a ranked model of $\KB$ (Lemma~\ref{Lemma:TerminationAndSoundness}, Item~3, plus Proposition~\ref{Prop:SatisfiesKB}), and so if $\alpha\in\Cn{0}{\KB}$, then $\alpha\in\RM^{*}_{\KB}$. 
{\bf P4} is an immediate consequence of the satisfaction of {\bf P7}.\footnote{%
For this conclusion we need the requirement (specified in Section \ref{Preliminaries})
that $\Prp$ contains at least two elements.}
{\bf P5} is an immediate consequence of the satisfaction of {\bf P6}. The latter holds by definition of $\Cn{\LM}{\KB}$. For {\bf P7}, see Section~\ref{RationalClosure}. {\bf P9} is an immediate consequence of the satisfaction of {\bf P7}.

Now consider {\bf P10}. From right to left, it is an immediate consequence of~{\bf P3}. From left to right, assume there is a formula $\typ\top\imp\alpha$ that is in $\Cn{\LM}{\KB}$, but not in $\Cn{0}{\KB}$. It means that there is a ranked model $\RM$ of $\KB$ that has in its lower layer a propositional valuation $v$ \st\ $v\sat\neg\alpha$; but, given that the model~$\RM^{*}_{\KB}$ defining $\Cn{\LM}{\KB}$ is the LM-minimum model of $\KB$, then also the lower layer of $\RM^{*}_{\KB}$ must contain the valuation $v$, against the hypothesis.

Failure of {\bf P8} can be seen in Example~\ref{rceg}. There we have $\lnot\p\in\Cn{\LM}{\KB}$ (there is no penguin) because $\lnot\p$ holds in both valuations occurring in~$\RM^{*}_{\KB}$. Thus LM-entailment forces us to infer $\lnot\p$ from $\KB$. But $\lnot\p\not\in\Cn{0}{\KB}$, because there does exist a ranked model $\RM$ of $\KB$ for which $\states{\p}{\RM}\neq \emptyset$, for instance the model $\RM_2$ appearing in Example~\ref{ex02} below. 
\end{proof}
\myskip

In summary then, LM-entailment satisfies all our postulates, except for Strict Entailment ({\bf P8}). Lest this be seen as a negative result, bear in mind that LM-entailment satisfies Conditional Strict Entailment ({\bf P9}), the weakened version of Strict Entailment, and therefore also Classical Entailment.

In the next section we turn to a form of entailment satisfying Strict Entailment, but at the price of having to forego Conditional Rationality, and therefore the Single Model postulate as well.

\section{PT-entailment}\label{PT-entailment}

In this section we consider another option for entailment based on a version of minimality, and derived from the characterisation of rational closure by Giordano~\etal.~\cite{GiordanoEtAl2012, GiordanoEtAl2015}.
The general idea is to respect the principle of \emph{presumption of typicality} (see Section~\ref{entailsection}), We shall refer to this form of entailment as \emph{Presumption of Typicality} entailment, shortened to \emph{PT-entailment}. Such a principle indicates the way in which  the property (RM) should be satisfied. If we have $\alpha\twiddle\gamma$ in our knowledge base \KB, then, in order to satisfy (RM), we have to add either $\alpha\twiddle\neg\beta$ or $\alpha\wedge\beta\twiddle\gamma$. The presumption of typicality requires that, whenever possible, we prefer the latter (that corresponds to a constrained application of monotonicity) over the former. 
Semantically, given the ranked models of a knowledge base~$\KB$, this corresponds to considering only those models in which every valuation is taken as typical as possible, that is, it is `pushed downward' in the model as much as possible, modulo the satisfaction of~$\KB$. 


In order to identify the interpretations that are necessary for the definition of a notion of entailment, we introduce a preference relation $\mpref{\PT}$ on the set of ranked interpretations that follows directly from the presumption of typicality.

%
%

\begin{definition}[Relation $\mpref{\PT}$]\label{preferred2}
For two ranked interpretations $\RM_1$ and $\RM_2$, $\RM_1\mpref{\PT} \RM_2$ if and only if for every $w\in \U$, $\RM_1(w)\leq\RM_2(w)$. $\RM_1\mprefstrict{\PT} \RM_2$ if and only if $\RM_1\mpref{\PT} \RM_2$ and not $\RM_2\protect{\mpref{\PT}} \RM_1$.
\end{definition}

%

It is easy to check that $\mpref{\PT}$ is a pre-order. Consistent with the principle of presumption of typicality, as a guideline in the choice of the relevant interpretations, the relation $\mpref{\PT}$ can be used  to identify the relevant interpretations for the definition of a notion of entailment: we choose the models of~$\KB$ in which the valuations are presumed to be as typical as possible, that is, the relevant models are those that are in $\min_{\mpref{\PT}} \Mod{\KB}$. Then, $\KB$ entails $\alpha$ if and only if $\alpha$ holds in all the (preferred) models in $\min_{\mpref{\PT}}\Mod{\KB}$. We will sometimes refer to the models in $\min_{\mpref{\PT}}\Mod{\KB}$ as the {\em PT-minimal} models of $\KB$. Note that if $\KB$ is unsatisfiable, it has exactly one PT-minimal model, namely the ranked interpretation $\RM$ for which $\U^{\RM}=\emptyset$. 


If we consider knowledge bases composed only of classical non-monotonic conditionals $\alpha\twiddle\beta$, Giordano et al. have proved that for every satisfiable knowledge base there is a unique PT-minimal model \cite[Theorem 1]{GiordanoEtAl2015}, and that such a PT-minimal model characterizes the rational closure of the knowledge base \cite[Theorem 2]{GiordanoEtAl2015}. Given such results, it is quite immediate to prove that, given a satisfiable conditional knowledge base, its PT-minimal model corresponds to the LM-minimal model.

\begin{proposition}\label{Prop:cond_LM=PT}
Let $\KB$ be a satisfiable conditional knowledge base. A ranked interpretation $\RM$ is $\KB$'s PT-minimal model iff it is $\KB$'s LM-minimal model.
\end{proposition}

\begin{proof}
If $\KB$ is a satisfiable conditional knowledge base, then it has a unique LM-minimal model $\RM$ (see Proposition \ref{Prop:uniqueLM}) and a unique PT-minimal model $\RM'$ \cite[Theorem 1]{GiordanoEtAl2015}. $\RM$ and $\RM'$ are equivalent, in the sense that they satisfy exactly the same conditionals, since they both characterise the rational closure of $\KB$ (see Proposition \ref{Prop:LMisRC} here for LM-minimality and the theorem by Giordano and others for PT-minimality \cite[Theorem 2]{GiordanoEtAl2015}). 

In order to show that they are exactly the same model, we just need to prove that whenever two ranked interpretations $\RM$ and $\RM'$ satisfy exactly the same set of conditionals, then they are the same interpretation. Let $\RM=(L_0, \ldots, L_{n-1},L_{\infty})$ and $\RM'=(M_0, \ldots, M_{n-1},M_{\infty})$. 

First of all, we prove $\U^{\RM}=\U^{\RM'}$: let $v\in \U^{\RM}$ and $v\notin \U^{\RM'}$, and let $\overline{v}$ be the characteristic formula of the valuation $v$; we would have $\RM'\sat \overline{v}\twiddle\bot$ and $\RM\not\sat \overline{v}\twiddle\bot$, against the hypothesis that $\RM$ and $\RM'$ satisfy  the same set of conditionals. $\U^{\RM}=\U^{\RM'}$ immediately implies that $L_{\infty}=M_{\infty}$.

We conclude the proof by induction on the rank of the cells below $\infty$. Given a cell $L_i=\{v_1,\ldots,v_n\}$, let $\overline{L_i}\defined(\overline{v_1}\lor\ldots\lor \overline{v_n})$.

Assume $L_0\neq M_0$, that is, w.l.o.g., there is a $v$ s.t. $v\in L_0$ and $v\notin M_0$. That implies $\RM'\sat\top\twiddle \neg \overline{v}$, while $\RM\not\sat\top\twiddle \neg \overline{v}$, against the hypothesis.

Given a number $j\leq (i-1)$, let $L_k=M_k$ for every $k$ s.t. $0\leq k< j$, but $L_j\neq M_j$, that is, w.l.o.g., there is a $v$ s.t. $v\in L_j$ and $v\notin M_j$.  That implies $\RM'\sat\neg(\bigvee_{0\leq k<j}\{\overline{L_k}\})\twiddle \neg \overline{v}$, while $\RM\not\sat\neg(\bigvee_{0\leq k<j}\{\overline{L_k}\})\twiddle \neg \overline{v}$, against the hypothesis. Since all their cells must contain the same valuations, $\RM$ and $\RM'$ are the same model.
\end{proof}

 Despite Proposition \ref{Prop:cond_LM=PT}, given the extra expressive power of~PTL, we obtain the surprising result that the two semantic constructions are not equivalent anymore. Moreover, in the present context, this notion of minimality can give back a number of minimal models, as the following example shows.

\begin{example}\label{ex02}
Consider the knowledge base $\KB$ from Example~\ref{rceg}. Then, one can see that $\min_{\mpref{\PT}}\Mod{\KB}=\{\RM_{1},\RM_{2},\RM_{3}\}$, where:

 



\vspace{0.3cm}

\noindent$\RM_1:$ \hspace{0.6cm}
\begin{tabular}{ | c|c| } 
 \hline
 $0$ & 
 $\{\lnot\fly,\lnot\p,\lnot\rob\}$, 
 $\{\fly,\lnot\p,\lnot\rob\}$, 
 \\ 
 \hline
 \end{tabular}
 \vspace{0.5cm}
 
 \noindent$\RM_2:$ \hspace{0.6cm}
\begin{tabular}{ | c|c| } 
\hline
 $2$ & 
 $\{\fly,\p,\lnot\rob\}$
 \\ 
 \hline
 $1$ & 
  $\{\lnot\fly,\lnot\p,\lnot\rob\}$,
 $\{\lnot\fly,\p,\lnot\rob\}$
 \\ 
 \hline
 $0$ & 
 $\{\fly,\lnot\p,\lnot\rob\}$
 \\ 
 \hline
 
\end{tabular}
 \vspace{0.5cm}

 \noindent$\RM_3:$ \hspace{0.6cm}
\begin{tabular}{ | c|c| } 
\hline
 $2$ & 
 $\{\lnot\fly,\lnot\p,\rob\}$
 \\ 
 \hline
 $1$ & 
  $\{\fly,\lnot\p,\rob\}$,
 $\{\fly,\lnot\p,\lnot\rob\}$
 \\ 
 \hline
 $0$ & 
 $\{\lnot\fly,\lnot\p,\lnot\rob\}$
 \\ 
 \hline
 
\end{tabular}

\end{example}
\myskip

In Example~\ref{ex02}, note that $\RM_1$ is the $\LM$-minimum of $\KB$. In fact, it is easy to check from the characterisation of rational closure in Section~\ref{RationalClosure} and Definition~\ref{preferred2} that the $\LM$-minimum of~$\KB$ is always in $\min_{\mpref{\PT}}\Mod{\KB}$. 

\begin{restatable}{proposition}{LMinmin}\label{LMinmin}
For every knowledge base $\KB$, the $\LM$-minimum of $\KB$ is in $\min_{\mpref{\PT}}\Mod{\KB}$.
\end{restatable}


\begin{proof}
Consider the definition of the preference relation for $\LM$-minimality.
\[
\begin{array}{rcl}
\RM_1 \mpref{\LM} \RM_2\ \textrm{if and only if} & \textit{either} &
L_i = M_i\ \textrm{for all } i\in\{0,\ldots,n-1,\infty\},\\
& \textit{or} & L_j \supseteq M_j \textrm{ for the smallest } j\geq 0 \textrm{ s.t. }L_j \neq M_j,
\end{array}
\]
where $\RM_1 = (L_0, \ldots, L_{n-1}, L_\infty)$ and $\RM_2 = (M_0, \ldots, M_{n-1}, M_\infty)$.
The result follows from the fact that if $\RM_1\mprefstrict{\PT} \RM_2$ then $\RM_1\mprefstrict{\LM} \RM_2$. 
To see that this holds, assume $\RM_1\mprefstrict{\PT} \RM_2$. Then $\RM_1(w) \leq \RM_2(w)$ for all $w \in \U$, with $\RM_1(w') < \RM_2(w')$ for at least one $w' \in \U$. From the latter, we know we cannot have $L_i = M_i$ for all $i$, so let $j \geq 0$ be minimal such that $L_j \neq M_j$. To show the conclusion $\RM_1\mprefstrict{\LM} \RM_2$ we must show $L_j \supseteq M_j$, so let $u \in M_j$. Then $\RM_2(u) = j$. Since $\RM_1\mprefstrict{\PT} \RM_2$ we know $\RM_1(u) \leq j$. But if $\RM_1(u) = k <j$ then $u \in L_k = M_k$ (by minimality of $j$), contradicting $u \in M_j$. Hence $\RM_1(u) = j$, i.e., $u \in L_j$ as required.

Knowing that $\RM_1\mprefstrict{\PT} \RM_2$ implies $\RM_1\mprefstrict{\LM} \RM_2$, it is easy to conclude our proof. Let $\RM$ be the LM-minimum of $\KB$, but not an element of $\min_{\mpref{\PT}}\Mod{\KB}$. That is,  
there is an $\RM^*\in\Mod{\KB}$ s.t. $\RM^*\mprefstrict{\PT} \RM$, that implies $\RM^*\mprefstrict{\LM} \RM$, thus contradicting the LM-minimality of $\RM$. 
\end{proof}
\myskip

We are now ready for the definition of our second type of entailment.

\begin{definition}[PT-entailment]\label{Def:PT-Entailment}
Let $\KB\subseteq\typup{\Lang}$ and $\alpha\in\typup{\Lang}$. We say $\KB$ \textbf{PT-entails} $\alpha$, denoted $\KB\NMentails_{\PT}\alpha$, if and only if $\min_{\mpref{\PT}}(\Mod{\KB})\subseteq\Mod{\alpha}$.
\end{definition}

Its corresponding consequence operator $\Cn{\PT}{\cdot}$ is inferentially {\em weaker} than $\Cn{\LM}{\cdot}$, since it is defined on a possibly larger set of models.

\begin{restatable}{proposition}{basicpostulates}\label{basicpostulates}
$\Cn{\PT}{\cdot}$ satisfies~{\bf P1}--{\bf P4} and {\bf P7}--{\bf P10}.
\end{restatable}
\begin{proof}
\noindent{\bf P1.} $\Cn{\PT}{\KB}$ is defined using only models of~$\KB$.



\noindent{\bf P2.} Since $\KB\subseteq \KB'\subseteq \Cn{\PT}{\KB}$, we have $\min_{\mpref{\PT}}\Mod{\KB})\subseteq\Mod{\Cn{\PT}{\KB}}\subseteq\Mod{\KB'}\subseteq\Mod{\KB}$. It is sufficient to prove that $\min_{\mpref{\PT}}\Mod{\KB'})=\min_{\mpref{\PT}}\Mod{\KB})$. 

Let $\RM$ be a model of $\KB$ and $\KB'$ s.t. $\RM\in \min_{\mpref{\PT}}\Mod{\KB'}$ and $\RM\notin \min_{\mpref{\PT}}\Mod{\KB})$. That is, there must be a model $\RM'$ of $\KB$ s.t. $\RM'\mprefstrict{PT}\RM$ and $\RM'\in \min_{\mpref{\PT}}\Mod{\KB}$. However, since $\min_{\mpref{\PT}}\Mod{\KB})\subseteq\Mod{\KB'}$, $\RM'$ is also a model of $\KB'$ that is PT-preferred to $\RM$, that is, it cannot be the case that $\RM\in \min_{\mpref{\PT}}\Mod{\KB'}$.
Inversely, let $\RM$ be a model of $\KB$ and $\KB'$ s.t. $\RM\in \min_{\mpref{\PT}}\Mod{\KB}$ and $\RM\notin \min_{\mpref{\PT}}\Mod{\KB'})$. That is, there must be a model $\RM'$ of $\KB'$ s.t. $\RM'\mprefstrict{PT}\RM$ and $\RM'\in \min_{\mpref{\PT}}\Mod{\KB'}$. However, since $\KB\subseteq \KB'$, $\RM'$ is also a model of $\KB$ that is PT-preferred to $\RM$, that is, it cannot be the case that $\RM\in \min_{\mpref{\PT}}\Mod{\KB}$.

Hence, for every $\KB,\KB'$ s.t. $\KB\subseteq\KB'\subseteq \Cn{\PT}{\KB}$, it must be $\min_{\mpref{\PT}}\Mod{\KB}=\min_{\mpref{\PT}}\Mod{\KB}$, that implies {\bf P2}.

\noindent{\bf P3.} Every model in $\min_{\mpref{\PT}}{\Mod{\KB}}$ is by definition a ranked model of~$\KB$. Hence if $\alpha\in\Cn{0}{\KB}$, i.e., $\alpha$ is true in all ranked models of $\KB$, then it is true in all ranked models in $\min_{\mpref{\PT}}{\Mod{\KB}}$, i.e., $\alpha\in\Cn{\PT}{\KB}$.

\noindent{\bf P4.} It is an immediate consequence of the satisfaction of {\bf P7}.\footnote{%
As in Theorem \ref{Thm:LM-properties}, for this conclusion we need the requirement (specified in Section \ref{Preliminaries}) that $\Prp$ contains at least two elements.}

\noindent{\bf P7.} See the analagous result by Giordano~\etal.~\cite[Section 2.3.2]{GiordanoEtAl2015}; in particular Theorem 2, that implies that in case of a conditional KB the use of PT-minimality leads to a single minimal model, characterising Rational Closure.

\noindent{\bf P8.} Let $\alpha$ be a propositional formula \st\ $\alpha\notin \Cn{0}{\KB}$: then there is a ranked model $\RM$ of $\KB$ \st\ $\RM(v)<\infty$ for some $v$ \st\ $v\sat\neg\alpha$. Either $\RM$ is a \emph{PT}-minimal model of $\KB$ itself, or there is a \emph{PT}-minimal model $\RM'$ of $\KB$ \st\ $\RM'\mpref{\PT}\RM$; that is, it must be the case that $\RM'(v)<\infty$ for some model $\RM'\in\min_{\mpref{\PT}}{\Mod{\KB}}$, that in turn implies that $\alpha\notin \Cn{\PT}{\KB}$.

\noindent{\bf P9.} It is an immediate consequence of the satisfaction of {\bf P7}, as explained in Section \ref{TowardsEntailment}, immediately after introducing {\bf P9}.

\noindent{\bf P10.} It is a direct consequence of Proposition~\ref{LMinmin} and the satisfaction of {\bf P10} for LM-entailment.
\end{proof}
\myskip

Unfortunately, \emph{Conditional Rationality} ({\bf P5}) is not valid and therefore, neither is the Single Model postulate ({\bf P6}).

\begin{restatable}{theorem}{norational}\label{norational}
There is some $\KB$ such that the conditional induced by $\Cn{\PT}{\KB}$ is not a rational conditional.
\end{restatable}


To see this, consider Example \ref{ex02}: we have $\typ\lnot\p\limp\lnot\rob\in\Cn{\PT}{\KB}$ (typical non-penguins are not robins). 
This is because we have $\min_{\pref_{\RM_i}}\states{\lnot\p}{\RM_i}\subseteq\states{\lnot\rob}{\RM_i}$ for each $i = 1,2,3$. However both $\typ\lnot\p\limp\neg \fly\not\in \Cn{\PT}{\KB}$ and $\typ(\lnot\p\land\fly)\limp\lnot\rob\not\in\Cn{\PT}{\KB}$. The former holds because, e.g., $\min_{\pref_{\RM_1}}\states{\lnot\p}{\RM_1}\nsubseteq\states{\lnot\fly}{\RM_1}$, the latter because $\min_{\pref_{\RM_3}}\states{\lnot\p \wedge \fly}{\RM_3}\nsubseteq\states{\lnot\rob}{\RM_3}$. This means the rational monotonicity property (RM) is not satisfied. 

On the other hand, observe that $\lnot\p\notin\Cn{\PT}{\KB}$. Recall from the proof of Theorem \ref{Thm:LM-properties} that we used the fact that $\lnot\p\in\Cn{\LM}{\KB}$ to show that LM-entailment does not satisfy Strict Entailment ({\bf P8}).

\section{PT'-entailment}\label{PT'-entailment}

As we have shown above, relying on LM-minimality results in the loss of property {\bf P8} (Strict Entailment), while using PT-minimality results in the loss of the uniqueness of the minimal model  ({\bf P6}) and the rationality of our conditional reasoning ({\bf P5}). To summarise, on the one hand LM-minimality, failing to satisfy {\bf P8}, can potentially enforce classical propositional information that is not a necessary consequence of the knowledge base. On the other hand, PT-minimality can be inferentially too weak. In this section we consider a third possibility, aimed at strenghtening the inferential power while still preserving the satisfaction of {\bf P8}. This third proposal is based on using the same approach as in PT-minimality, but among the PT-minimal models we consider only the ``biggest'' ones, that is, the ones with the \emph{maximal} sets of possible valuations (\wrt~$\subseteq$). This should allow  us to augment the inferential power (we define the entailment relation using fewer models), while still preserving {\bf P8} (we consider all the biggest models, that is, the models that assume as little propositional knowledge as possible). We now analyse this option.

We  
let $\min_{\mpref{\PT}}^\supseteq\Mod{\KB}\defined \{\RM\in\min_{\mpref{\PT}}\Mod{\KB} \mid$ 
there is no $\RM'\in\min_{\mpref{\PT}}\Mod{\KB} 
\text{ \st~}\U^{\RM'}\supset\U^{\RM}\}$.
\myskip

The corresponding entailment relation $\NMentails_{\PT'}$ can be defined as follows.

\begin{definition}[PT'-entailment]\label{Def:PT'-Entailment}
Let $\KB\subseteq\typup{\Lang}$ and $\alpha\in\typup{\Lang}$. We say $\KB$ \textbf{PT'-entails} $\alpha$, denoted $\KB\NMentails_{\PT'}\alpha$, if and only if $\min_{\mpref{\PT}}^\supseteq\Mod{\KB}\subseteq\Mod{\alpha}$.
\end{definition}
\myskip

For example, in Example \ref{ex02} we would consider only $\RM_2$ and $\RM_3$. 

Note that if $\KB$ is unsatisfiable
then $\min_{\mpref{\PT}}^\supseteq\Mod{\KB}$ is a singleton set containing the ranked interpretation $\RM$ for which $\U^{\RM}=\emptyset$. Also,  recall from Section \ref{PT-entailment} that for every satisfiable conditional knowledge base $\KB$ there is a single PT-minimal model \cite[Theorem 1]{GiordanoEtAl2015}, that characterises  the rational closure of $\KB$ \cite[Theorem 2]{GiordanoEtAl2015}. Such a single PT-minimal model is by definition also the only PT'-minimal model of $\KB$, and consequently, in case of conditional knowledge bases, PT'-entailment also characterises the rational closure.

Our first result regarding PT'-entailment is that it is inferentially stronger than PT-entailment.

\begin{proposition}\label{supset}
{\em (i)} For every $\KB \subseteq \typup{\Lang}$ and every $\alpha\in\typup{\Lang}$, if $\KB\NMentails_{\PT}\alpha$ then $\KB\ \NMentails_{\PT'}\alpha$. {\em (ii)} There exists some $\KB' \subseteq \typup{\Lang}$ and $\beta\in\typup{\Lang}$ such that $\KB'\ \NMentails_{\PT'}\beta$ and $\KB'\ \not\NMentails_{\PT}\beta$.
\end{proposition}

\begin{proof}
{\em (i)}. Note  that, since $\min_{\mpref{\PT}}^\supseteq\Mod{\KB}\subseteq\min_{\mpref{\PT}}\Mod{\KB}$ for every $\KB$,  $\NMentails_{\PT}\subseteq \NMentails_{\PT'}$. 
{\em (ii)}. Observe from  Example \ref{ex_PT'}, here below, that $\KB'\NMentails_{\PT'}\typ \top\rightarrow \neg \fly$ but $\KB'\not\NMentails_{\PT}\typ \top\rightarrow \neg \fly$.
\end{proof}

\begin{example}\label{ex_PT'}
Consider the knowledge base $\KB'\defined\{\typ{\top}\limp(\lnot\p\land\lnot\rob), \typ\p\limp\lnot\fly,\typ\rob\limp\typ\fly, \p\limp \neg \rob\}$, which is a modified version of the knowledge $\KB$ from Example~\ref{rceg}. The only difference is that now we state that typical penguins are non-flying birds, not that they are typical non-flying birds.

Then, one can check that $\min_{\mpref{\PT}}\Mod{\KB'}=\{\RM_{1},\RM_{2}\}$, where:

\vspace{0.3cm}

\noindent$\RM_1:$ \hspace{0.6cm}
\begin{tabular}{ | c|c| }
\hline
 $2$ & 
 $\{\fly,\p,\lnot\rob\}$
 \\ 
 \hline
 $1$ & 
  
 $\{\lnot\fly,\p,\lnot\rob\}$
 \\ 
 \hline
 $0$ & 
 $\{\lnot\fly,\lnot\p,\lnot\rob\}$, 
 $\{\fly,\lnot\p,\lnot\rob\}$, 
 \\ 
 \hline
 \end{tabular}
 \vspace{0.5cm}

 \noindent$\RM_2:$ \hspace{0.6cm}
\begin{tabular}{ | c|c| } 
\hline
 $2$ & 
 $\{\lnot\fly,\lnot\p,\rob\}$, $\{\fly,\p,\lnot\rob\}$
 \\ 
 \hline
 $1$ & 
  $\{\fly,\lnot\p,\rob\}$,
 $\{\fly,\lnot\p,\lnot\rob\}$, $\{\lnot\fly,\p,\lnot\rob\}$
 \\ 
 \hline
 $0$ & 
 $\{\lnot\fly,\lnot\p,\lnot\rob\}$
 \\ 
 \hline
 
\end{tabular}

\vspace{0.3cm}

while $\min_{\mpref{\PT}}^\supseteq \Mod{\KB'}=\{\RM_{2}\}$, since $\U^{\RM_{2}}\supset\U^{\RM_{1}}$.

\end{example}

Unfortunately, while PT'-entailment is an improvement over PT-entailment in terms of inferential strength, it is weaker than PT-entailment when it comes to the satisfaction of the list of desirable properties. That is, it satisfies, and fails to satisfy, the same properties as PT-entailment, except for Typical Entailment ({\bf P10}), which PT-entailment satisfies, but PT'-entailment does not.

\begin{restatable}{proposition}{basicpostulates_PT'}\label{basicpostulates_PT'}
$\Cn{\PT'}{\cdot}$ satisfies~{\bf P1}--{\bf P4} and {\bf P7}--{\bf P9}, but does not satify {\bf P5}, {\bf P6}, and {\bf P10}.
\end{restatable}

\begin{proof}
\noindent Regarding {\bf P1}, {\bf P2}, {\bf P3}, {\bf P4}, and {\bf P9} the proof for $\Cn{PT'}{\cdot}$ is the same as for $\Cn{PT}{\cdot}$ (Proposition \ref{basicpostulates} above).

\noindent Regarding the failure of {\bf P5}, consider Example \ref{ex02}. In this example, while $\min_{\mpref{\PT}}\Mod{\KB}=\{\RM_{1},\RM_{2},\RM_{3}\}$, we have that $\min_{\mpref{\PT}}^\supseteq \Mod{\KB}=\{\RM_{2},\RM_{3}\}$. We can use the same case used in the proof of Theorem \ref{norational}: we have $\KB\NMentails_{\PT'} \typ(\lnot \p)\limp \lnot \rob$, but neither $\KB\NMentails_{\PT'} \typ(\lnot \p)\limp \lnot \fly$, nor $\KB\NMentails_{\PT'} \typ(\lnot \p\land \fly)\limp \lnot \rob$ hold.

\noindent The failure of {\bf P5} immediately implies the failure of {\bf P6}.

\noindent{\bf P7.} As pointed out in Proposition \ref{basicpostulates}, in case we are dealing with a conditional KB, deciding \emph{PT}-minimality over a satisfiable conditional KB gives back a single minimal model, characterising Rational Closure. It follows immediately that such a model is also the only  \emph{PT'}-minimal one.

\noindent{\bf P8.} Again, it follows from the satisfaction of {\bf P8} for \emph{PT}-entailment (see Proposition \ref{basicpostulates}). Let $\KB$ be a knowledge base and $\alpha$ be a propositional formula. If there is a \emph{PT}-minimal model $\RM$ s.t. $\RM(v)\leq\infty$ for some $v\not\sat\alpha$, then, by definition, there must be  also in $\min_{\mpref{\PT}}^\supseteq\Mod{\KB}$ a model $\RM'$ of $\KB$ s.t. $\RM'(v)\leq\infty$.

\noindent For the failure of {\bf P10}, we consider Example \ref{ex_PT'} and the case used in the proof of Proposition \ref{supset}: $\KB'\NMentails_{\PT'}\typ \top\rightarrow \neg \fly$ but, since $\KB'\not\NMentails_{\PT}\typ \top\rightarrow \neg \fly$ and $\NMentails_{\PT}$ satisfies \emph{Ampliativeness} ({\bf P3}), $\typ \top\rightarrow \neg \fly$ is not in $\Cn{0}{\KB'}$.

\end{proof}
\myskip



\section{Making sense of the impossibility result}\label{Weakening}

Theorem~\ref{Impossibility} in Section~\ref{TowardsEntailment} shows that there is no PTL consequence operator satisfying all of our postulates---more specifically, none satisfying {\bf P1}, {\bf P2}, {\bf P3}, {\bf P5}, {\bf P8}, and {\bf P10}. This raises the important question of which of these postulates ought to be foregone in the search for an appropriate form of PTL entailment. In trying to find an answer to this question, it is useful to consider the three forms of entailment we proposed in the previous sections. 
The answer seems to be that it makes sense to consider (at least) two forms of entailment for PTL, represented here by LM-entailment and PT-entailment. PT'-entailment is not viewed as a viable option, given that it satisfies fewer properties than PT-entailment. In essence then, it comes down to a choice between having a form of entailment  that satisfies Strict Entailment
(PT-entailment), and one  that satisfies the Single Model postulate and Conditional Rationality, \ie, is based on a rational conditional (LM-entailment).

The advantage of LM-entailment is that it satisfies all postulates except for Strict Entailment, which includes not only Single Model and Conditional Rationality, but also Conditional Strict Entailment and Classical Entailment, the weakened versions of Strict Entailment.  On the other hand, the argument for PT-entailment is that the Single Model property is too restrictive in the context of full PTL, and ought to be dropped. That is, in a logic as expressive as PTL in which there are not any restrictions on the typicality operator, any form of entailment based on minimality, and adhering to the presumption of typicality, as outlined in Section~\ref{PT-entailment}, is likely to violate the Single Model property. 

The point of view that different forms of entailment can be appropriate in enriched versions of propositional logic, particularly enriched versions dealing with aspects of typicality, is not surprising, nor new. Lehmann~\cite{Lehmann1995} makes the case for two forms of entailment for the conditional logic discussed in Section~\ref{KLM} on which PTL is based. He draws a distinction between \emph{prototypical reasoning}, corresponding to rational closure as discussed in Section~\  \ref{RationalClosure}, and \emph{presumptive reasoning}. 

The intuition underlying prototypical reasoning is that conclusions to be drawn are constrained by the typicality of the objects under consideration. To make matters more concrete, suppose we know that birds typically fly, that birds typically have wings, that robins are birds, that penguins are birds, and that penguins typically don't fly. Robins can be regarded as typical birds and therefore inherit the properties of typical birds, such as having wings. Penguins, on the other hand, should be regarded as atypical birds since they typically cannot fly, and therefore do not inherit the properties of a typical bird, such as having wings. This is to be contrasted with presumptive reasoning, a more permissive form of reasoning for which the intuition is to draw conclusions unless we have specific information to the contrary. In our example above presumptive reasoning would allow us to conclude that penguins typically have wings (since we do not have explicit information contradicting that conclusion), thereby distinguishing it from prototypical reasoning. 

Our argument here is \emph{not} that the relationship between PT-entailment and LM-entailment is analogous to the relationship between prototypical reasoning and presumptive reasoning, although it is true that LM-entailment can be viewed as a refinement of PT-entailment (yielding more conclusions), just as presumptive reasoning is a refinement of prototypical reasoning. Rather, the important point is that differences in context will determine which form of entailment is appropriate. It is our contention that the same principle applies to the differences between LM-entailment and  PT-entailment. Below we discuss the technical differences between the two forms of entailment and then provide an example to illustrate the principle.

As we have seen above, the difference between these two forms of entailment comes down to a choice between Strict Entailment on the one hand, Conditional Rationality (and Single Model) on the other hand. Employing LM-entailment ensures that we remain rational (\ie, satisfying all the KLM properties), but at the cost of going beyond Tarskian monotonicity for typicality-free sentences. Conversely, making use of PT-entailment allows us to remain Tarskian for typicality-free sentences, but forces us to forego rationality, and in particular, the rational monotonicity property RM. Intuitively then, LM-entailment is the more permissive form of entailment here. Not only do we remain rational, unlike PT-entailment, but we do so at the cost of allowing the entailment of more typicality-free sentences than permitted by PT-entailment.  The example below is indicative of the factors to take into account when deciding, in a specific context, which of LM-entailment or PT-entailment is the more appropriate form of reasoning.


\begin{example}\label{LM-PT}
Consider again the knowledge base $\KB \defined\{\typ{\top}\limp(\lnot\p\land\lnot\rob), \typ\p\limp\typ\lnot\fly,\typ\rob\limp\typ\fly, \p\limp\neg \rob\}$ introduced in Example \ref{rceg}. From Examples \ref{rceg} and \ref{ex02} it is not hard to verify that both LM-entailment and PT-entailment sanction the conclusion that typical non-robins are not penguins ($\KB\NMentails_{\LM}\typ(\lnot\rob)\limp\lnot\p$ and 
$\KB\NMentails_{\PT}\typ(\lnot\rob)\limp\lnot\p$), and do \emph{not} allow for the entailment that typical non-robins can fly 
($\KB\nNMentails_{\LM}\typ(\lnot\rob)\limp\fly$ and 
$\KB\nNMentails_{\PT}\typ(\lnot\rob)\limp\fly$). This leaves us with a choice. On the one hand it is reasonable to conclude from this that typical \emph{non-flying} non-robins are not penguins (that is, $\typ(\lnot\rob\land \lnot \fly)\limp\lnot\p$). In fact, rational monotonicity requires of us to be able to draw this conclusion. But in order to do so, we need to be able to conclude that there are no penguins, which violates Strict Entailment. This is the route followed by LM-entailment. The other option would be to insist that we do not have enough information to conclude that there are no penguins, but in the process of doing so, also forego the conclusion that typical non-flying non-robins are not penguins. That is, we insist on Strict Entailment at the expense of rational monotonicity. This is the path followed by PT-entailment. \hfill \qed
\end{example}

\section{Related work}\label{RelatedWork}

\newcommand{\nec}{\Box}
\newcommand{\ALC}{\ensuremath{\mathcal{ALC}}}
\newcommand*{\myleftmid}{%
      \mathrel{\vcenter{\offinterlineskip
      \vskip-0.25ex\hbox{$\shortmid$}}}}
\newcommand*{\myrightmid}{%
      \mathrel{\vcenter{\offinterlineskip
      \vskip-0.7ex\hbox{$\shortmid$}}}}
\newcommand*{\twosim}{%
      \mathrel{\vcenter{\offinterlineskip
      \vskip0.05ex\hbox{$\sim$}\vskip0.25ex\hbox{$\sim$}}}}
\newcommand{\bartwosim}{\mathrel{\myleftmid}\hskip-0.07ex\joinrel\twosim}
\newcommand{\flag}{\mathrel{\bartwosim}\hskip-.06ex\joinrel\myrightmid\hskip-0.55ex}
\newcommand{\flame}{\scalebox{0.8}{\raisebox{-0.2ex}{\rotatebox{57}{\ensuremath{\flag}}}}}

To the best of our knowledge, the first attempt to formalise an explicit notion of typicality in defeasible reasoning was that by Delgrande~\cite{Delgrande1987}. Given the strong links between our constructions and the KLM approach, most of the remarks in the comparison made by Lehmann and Magidor~\cite[Section 3.7]{LehmannMagidor1992} are applicable in comparing Delgrande's approach to ours and therefore we shall not repeat them here.
\myskip

Crocco and Lamarre~\cite{CroccoLamarre1992} as well as Boutilier~\cite{Boutilier1994} have explored the links between conditionals and notions of normality similar to the one we investigate here. In particular, Boutilier defines a family of conditional logics of normality in which a statement of the form ``if $\alpha$, then normally~$\beta$'' is formalised via a binary modality~$\Rightarrow$ as a conditional~$\alpha\Rightarrow\beta$. Here we achieve the same with a unary operator.
	
Roughly speaking, Boutilier's semantic intuition is the same as that of~KLM (and therefore the same as ours). The main difference is that Boutilier defines a conditional connective~$\Rightarrow$ in the language, whereas Kraus~\etal.\ define~$\twiddle$ at a meta-level to the language. In this respect, Boutilier's approach is more general in that it allows for nested conditionals. If these are omitted, \ie, if one works in the `flat' conditional logic in which~$\Rightarrow$ is the main connective and no nesting is allowed, then one gets the same results for preferential entailment with both systems. So Boutilier achieves with modalities (he works in a bi-modal language) what Kraus~and colleagues achieve with a (meta-level) preference order.

It turns out that in Boutilier's approach one cannot always capture the notion of ``most typical $\alpha$'s''. In Boutilier's modal logic, such a set (of \emph{most} normal $\alpha$-worlds) need not exist in general. This is because Boutilier drops the smoothness condition~\cite[p.\ 103]{Boutilier1994} and therefore at any point in a ranked model one can have infinitely descending chains of increasingly more normal $\alpha$-worlds. If one imposes smoothness in Boutilier's approach, which can be done by \eg\ requiring the ordering determined by Boutilier's~$\nec$ also to be \emph{Noetherian}, one could then define his conditional~$\Rightarrow$ more elegantly as follows:
\begin{equation}\label{Equation:ArrowAsTyp}
\alpha\Rightarrow\beta\defined\typ\alpha\limp\beta
\end{equation}
where, in Boutilier's notation, $\typ\alpha$ would be given by
\begin{equation}\label{Equation:TypInS4}
\typ\alpha\defined\alpha\land\nec\lnot\alpha
\end{equation}
(Of course negated conditionals of the form~$\alpha\not\Rightarrow\beta$ can then be expressed as~$\lnot(\typ\alpha\limp\beta)$.) In adopting smoothness and defining conditionals in this way, one would expect both approaches to become equivalent modulo the underlying language --- ours is propositional, whereas Boutilier's is modal. However, our statement~$\typ\alpha\limp\beta$ differs from Boutilier's~$\alpha\Rightarrow\beta$ in a significant way. In Boutilier's approach, a statement of the form~$\alpha\Rightarrow\beta$ is true at some world (in a ranked model) if and only if it is true at \emph{all} worlds in that ranked model~\cite[p.\ 114]{Boutilier1994}. On the other hand, it is not hard to find a ranked model in which~$\typ\alpha\limp\beta$ holds at a world without being true in the whole model. This establishes Boutilier's conditional as a `global' statement, while ours has the (more general) `local flavour'. We can easily simulate Boutilier's notion of \emph{acceptance}~\cite[p.\ 115]{Boutilier1994} by stating~$\top\limp(\typ\alpha\limp\beta)$.

It is also worth mentioning that our interpretation of the conditional~$\Rightarrow$ in~(\ref{Equation:ArrowAsTyp}) above and Boutilier's differ in another subtle way, which also relates to whether one adopts smoothness or not. In~(\ref{Equation:ArrowAsTyp}), $\alpha\Rightarrow\beta$ is defined as ``the normal $\alpha$'s are~$\beta$'s'', whereas, strictly speaking, Boutilier's definition of $\alpha\Rightarrow\beta$ reads as ``there is a point from which $\alpha\limp\beta$ is not violated''. Such a `frontier' for normality, implicitly referred to in Boutilier's definition of $\alpha\Rightarrow\beta$, is not as crisp as ours in the sense that the point where one draws the normality line might be too `far away' (in the ordering) from the more and more normal $\alpha$-worlds. One can definitely make a case for dropping the smoothness condition, but requiring it is a small price to pay given the much simpler account of typicality one obtains.

When it comes to entailment from a defeasible knowledge base, all approaches discussed above adopt a Tarskian-style notion of consequence and therefore do not go beyond ranked entailment. The move towards a non-monotonic notion of entailment and an investigation of its different facets in the context of~PTL was precisely our motivation in the present work.

Giordano~\etal.~\cite{GiordanoEtAl2009c} proposed the system $P_{min}$ which is based on a language that is as expressive as the one we propose in this paper. However, they end up using a constrained form of such a language that goes only slightly beyond the expressivity of the language of KLM-style conditionals (their \emph{well-behaved knowledge bases}). More importantly, their approach differs from ours since they build $P_{min}$ on a semantic approach that relies on preferential models and a notion of minimality that is more akin to circumscription~\cite{McCarthy1980}.

In a description logic setting, Giordano~\etal.~\cite{GiordanoEtAl2007} also study notions of typicality. Semantically, they do so by placing an (absolute) ordering on \emph{objects} in first-order domains in order to define versions of defeasible subsumption relations in the description logic~\ALC. The authors moreover extend the language of~\ALC\ with an explicit typicality operator $\Typ(\cdot)$ of which the intended meaning is to single out instances of a concept that are deemed as `typical'. That is, given an~\ALC\ concept~$C$, $\Typ(C)$ denotes the most typical individuals having the property of being~$C$ in a particular~DL interpretation. It is worth pointing out, though, that most of the analysis in the work of Giordano~\etal. is dedicated to a constrained use of the typicality operator~$\Typ(\cdot)$, that is allowed to occur only in the left-hand side of GCIs and not in the scope of other concept constructors. Not having such a syntactic constraint is a feature of our approach that we have put forward in the present work. Still in the framework of Description Logics, also Bonatti~\etal. \cite{BonattiEtAl2015} introduce a typicality operator $N(\cdot)$, with a meaning that mirrors the operator $\Typ(\cdot)$; also the use of the $N$ operator is generally constrained, and their semantic framework is differs from the present one, not being preferential.
	
Giordano~\etal.'s~approach has been extended in a series of papers~\cite{GiordanoEtAl2008,GiordanoEtAl2009a,GiordanoEtAl2013,GiordanoEtAl2015}, in particular also to deal with the computation of non-monotonic entailment from defeasible knowledge bases. In the latter case, the authors define a hyper-tableau calculus to compute the rational closure of a defeasible ontology via a minimal-model construction~\cite{GiordanoEtAl2012,GiordanoEtAl2015} that, as mentioned before, is closely related to our notion of PT-entailment. Nevertheless, that remains the only notion of non-monotonic entailment investigated by the authors. We conjecture the more expressive~DL setting has the potential to give rise to a much broader spectrum of consequence relations when enriched with typicality operators, in particular when these apply not only to concepts but also to roles~\cite{Varzinczak2018}. Nevertheless, that remains the only notion of non-monotonic entailment investigated by the authors. We conjecture the more expressive~DL setting has the potential to give rise to a much broader spectrum of consequence relations when enriched with typicality operators, in particular when these apply not only to concepts but also to roles~\cite{Varzinczak2018}.
\myskip

Finally, Britz and Varzinczak~\cite{BritzVarzinczak2018-JANCL,BritzVarzinczak2018-JoLLI} investigate another, complementary aspect of defeasibility to the one here presented by introducing (non-standard) modal operators allowing us to talk about relative normality in accessible worlds. With their defeasible versions of modalities, namely~${\flag}$ and~${\flame}$, formalising respectively the notions of \emph{defeasible necessity} and \emph{distinct possibility}, it becomes possible to make statements of the form ``$\alpha$ holds in all of the normal (typical) accessible worlds'', thereby capturing defeasibility of what is `expected' in target worlds. (Note that this is different from stating something like~$\nec\typ\alpha$, which says that all accessible worlds are typical $\alpha$-worlds.) Such preferential versions of modalities allow for the definition of a family of modal logics in which defeasible modes of inference such as defeasible actions, knowledge and obligations can be expressed. These can be integrated either with existing $\twiddle$-based modal logics~\cite{BritzEtAl2011b,BritzEtAl2012} or with a modal extension of our typicality operator in striving towards a comprehensive theory of defeasible reasoning in more expressive languages.

\section{Conclusion}\label{Conclusion}

The focus of this paper is an investigation into the entailment problem for the logic PTL. We approached the problem from two angles: an abstract formal perspective, in which a set of appropriate postulates was presented and discussed, and a constructive perspective, in which three specific entailment relations were defined and studied. The primary conclusion to be drawn from this investigation is that a logic as expressive as PTL supports more than one form of entailment. This conclusion is supported from the abstract perspective via an impossibility result, as well as through the constructive approach via the definition of two of the three distinct types of PTL entailment: LM-entailment and PT-entailment. While both forms of entailment are generalisations of \emph{rational closure}, only one, LM-entailment, retains all the rationality properties associated with rational closure, formalised as the Conditional Rationality postulate~({\bf P5}). However, it does not satisfy  Strict Entailment~({\bf P8}), a postulate which requires an entailment relation to remain Tarskian for conclusions not involving typicality, although it satisfies weakened versions of Strict Entailment ({\bf P9} and {\bf P9$'$}).  On the other hand, the other form of entailment we studied, PT-entailment, satisfies~{\bf P8}, but not Conditional Rationality~({\bf P5}). 

The framework of Booth~\etal.~\cite{BoothEtAl2012,BoothEtAl2013} is, to the best of our knowledge, the first attempt to introduce a full-fledged typicality operator into propositional logic. In terms of other related work, the closest we are aware of is the restricted form of typicality for description logics by Giordano~\etal.~\cite{GiordanoEtAl2009b}. However, a consequence of their restricted use of typicality is that a propositional version of their logic would correspond to a KLM-style conditional logic in which rational closure behaves well, and which is much less expressive than PTL.

Britz~\etal.~\cite{BritzEtAl2009} and Giordano~\etal.~\cite{GiordanoEtAl2009b} have investigated the connection between the KLM approach and G\"odel-L\"ob modal logic, which is closely related to~PTL. Exploiting this connection should deliver an axiomatisation of an inference relation corresponding to ranked entailment, but it does not seem useful for modelling entailment relations based on minimisation as LM- and PT-entailment.

For future work, an obvious open question is whether our conjecture, that the subsets of postulates satisfied by LM-entailment and PT-entailment respectively provide appropriate abstract formalisations of two distinct forms of PTL entailment, can be formalised through representation theorems. From a computational perspective, it is worth investigating whether, as is the case for rational closure for conditional logics, the computation of (the different forms of) PTL entailment can be reduced to  a series of classical entailment checks.

Our results in the propositional setting pave the way for an investigation of appropriate forms of entailment in other, more expressive, preferential approaches, such as preferential description logics~\cite{BritzEtAl2011c,GiordanoEtAl2013,BritzEtAl2015a,BritzVarzinczak2016b,BritzVarzinczak2017-DL,CasiniEtAl2018,BritzEtAl2019} and modal logics of defeasibility~\cite{BritzEtAl2011b,BritzVarzinczak2013,BritzVarzinczak2017-JANCL,BritzVarzinczak2018-JoLLI}. The move to logics with more structure is of a challenging nature, and a simple rephrasing of our approach to these logics may not deliver the expected results. We are currently investigating these issues.

\section*{Acknowledgements}

The work of Giovanni Casini and of Thomas Meyer has received funding from the European Union's Horizon 2020 research and innovation programme under the Marie Sk{\l}odowska-Curie grant agreement No. 690974 (MIREL project). The work of Thomas Meyer has been supported in part by the National Research Foundation of South Africa (grant No. UID 98019).


\section*{Appendix}
\begin{appendices}
\newcommand{\M}{\mathscr{M}}
\section{Proofs for Section \ref{Preliminaries}}\label{App:Proof_Prel}

We give here a proof of Proposition \ref{Prop:LMisRC}. In order to do that, we need to introduce some extra notions and prove some extra propositions. First of all, analogously to the definitions for PTL introduced in Section \ref{BackgroundPTL}, we say that a set of conditionals $\mathcal{C}$ is {\em satisfiable} iff there is a ranked interpretation $\RM$ for which $\U^{\RM}\neq\emptyset$ satisfying all the conditionals in it, and let $\Mod{\mathcal{C}}$ be the set of all the ranked models of $\mathcal{C}$.

We are going to use a notion of {\em merging} ranked interpretations. 

\begin{definition}[Ranked Union]\label{Def:ranked_union}

Given a  set of ranked interpretations $\mathfrak{R}=\{\RM_1,\ldots,\RM_n\}$, its \textbf{ranked union} $\RM^{\mathfrak{R}}$ is defined as follows:

\begin{itemize}
\item for every $v,v'\in \U$, $v\prec_{\RM^{\mathfrak{R}}} v'$ iff $\min\{\RM_i(v)\mid \RM_i\in \mathfrak{R}\}<\min\{\RM_j(v')\mid \RM_j\in \mathfrak{R}\}$.
\item $\RM^{\mathfrak{R}}:=(L^{\RM^{\mathfrak{R}}}_0,\ldots,L^{\RM^{\mathfrak{R}}}_{n-1},L^{\RM^{\mathfrak{R}}}_{\infty})$ is defined as

    \begin{itemize}
        \item $L^{\RM^{\mathfrak{R}}}_{\infty}:=\bigcap \{L^{\RM_i}_{\infty}\mid \RM_i\in \mathfrak{R}\}$.
        \item $L^{\RM^{\mathfrak{R}}}_0:=\min_{\prec_{\RM^{\mathfrak{R}}}}(\U\setminus L^{\RM^{\mathfrak{R}}}_{\infty})$; $L^{\RM^{\mathfrak{R}}}_1:=\min_{\prec_{\RM^{\mathfrak{R}}}}(\U\setminus L^{\RM^{\mathfrak{R}}}_0\cup L^{\RM^{\mathfrak{R}}}_{\infty})$; and so on until $L^{\RM^{\mathfrak{R}}}_{n}=\emptyset$.
    \end{itemize}
\end{itemize}
\end{definition}

\begin{proposition}\label{Prop:ranked_union}
Let $\mathcal{C}$ be a satisfiable set of conditionals, and let $\mathfrak{R}:=\{\RM_1,\ldots,\RM_n\}$ be a set of models of $\mathcal{C}$. Then their ranked union $\RM^{\mathfrak{R}}$ is a model of $\mathcal{C}$, and $\RM^{\mathfrak{R}}\mprefstrict{\LM} \RM_i$ for every $\RM_i \in(\mathfrak{R}\setminus \RM^{\mathfrak{R}})$.
\end{proposition}

\begin{proof}
    We first prove that $\RM^{\mathfrak{R}}:=(L^{\RM^{\mathfrak{R}}}_0,\ldots,L^{\RM^{\mathfrak{R}}}_{n-1},L^{\RM^{\mathfrak{R}}}_{\infty})$ is a model of $\mathcal{C}$. For every $v\in L^{\RM^{\mathfrak{R}}}_0$, it must be the case, by Definition \ref{Def:ranked_union},  that $v\in L^{\RM_i}_0$ for some $\RM_i\in\mathfrak{R}$;  since such $\RM_i$ is a model of $\mathcal{C}$ and $v\in L^{\RM_i}_0$, $v\sat \neg\alpha\lor\beta$ for every $\alpha\twiddle \beta\in\mathcal{C}$. 
    
    Now, let $v\in L^{\RM^{\mathfrak{R}}}_i$, with $0<i<n$, s.t. $v\sat\alpha\land\neg\beta$ for some $\alpha\twiddle \beta\in\mathcal{C}$ (if there is no such $v$, $\RM^{\mathfrak{R}}$ is necessarily a model of $\mathcal{C}$). Being every $\RM_j\in\mathfrak{R}$ a model of $\mathcal{C}$, it must be the case, again by Definition \ref{Def:ranked_union}, that  $h_{\RM_j}(v)\geq i$ and $h_{\RM_j}(v')< h_{\RM_j}(v)$ for some $v'$ satisfying $\alpha\land\beta$. Hence,   it must be that  $\min\{\RM_i(v')\mid \RM_i\in \mathfrak{R}\}<\min\{\RM_j(v)\mid \RM_j\in \mathfrak{R}\}$ for some $v'$ satisfying $\alpha\land\beta$, that is, $v'\prec_{\RM^{\mathfrak{R}}} v$, that implies $v'\in L^{\RM^{\mathfrak{R}}}_j$, with $j<i$. \\ 
    To summarise, for every $\alpha\twiddle \beta\in \mathcal{C}$, if there is a valuation $v\in \RM^{\mathfrak{R}}$ s.t. $v\sat\alpha\land\neg\beta$, then there is a valuation $v'\in \RM^{\mathfrak{R}}$ s.t. $v'\sat\alpha\land\beta$ and $v'\prec_{\RM^{\mathfrak{R}}} v$; hence $\RM^{\mathfrak{R}}$ is a model of $\mathcal{C}$.
    
    
    Now we prove that $\RM^{\mathfrak{R}}\mpref{\LM} \RM_i$ for every $\RM_i \in\mathfrak{R}$. \\
    Let $\RM_i:=(L_0,\ldots,L_{n-1},L^{\RM^{\mathfrak{R}}}_{\infty})$ s.t. $\RM_i\in \mathfrak{R}$ and $\RM^{\mathfrak{R}}\not\mpref{\LM} \RM_i$. That is, there is an $i$ s.t. $L^{\RM^{\mathfrak{R}}}_i\not\subseteq L_i$, while $L^{\RM^{\mathfrak{R}}}_j= L_j$ for every $j<i$. That is, there is a $v\in L_i$ s.t. $v\notin L^{\RM^{\mathfrak{R}}}_i$. By definition of $\RM^{\mathfrak{R}}$, that implies that $v\in L^{\RM^{\mathfrak{R}}}_j$ for some $j<i$, but that cannot be the case, since $L^{\RM^{\mathfrak{R}}}_j= L_j$ for every $j<i$. Hence $\RM^{\mathfrak{R}}\mpref{\LM} \RM_i$ for every $\RM_i \in\mathfrak{R}$.
    
    We finish by proving that $\RM_i\not\mpref{\LM} \RM^{\mathfrak{R}}$ for every $\RM_i \in(\mathfrak{R}\setminus \RM^{\mathfrak{R}})$. \\
    Let $\RM_i$ be a model in $\mathcal{C}$ s.t. $\RM_i\mpref{\LM} \RM^{\mathfrak{R}}$. Since $\RM^{\mathfrak{R}}\mpref{\LM} \RM_i$, we must conclude that for every $i$, for  every cell $L^{\RM_i}_i$ composing $\RM_i$ and every cell $L^{\RM^{\mathfrak{R}}}_i$ composing $\RM^{\mathfrak{R}}$, $L^{\RM_i}_i=L^{\RM^{\mathfrak{R}}}_i$; that is, $\RM_i$ and $\RM^{\mathfrak{R}}$ are exactly the same model. Hence $\RM^{\mathfrak{R}}\mprefstrict{\LM} \RM_i$ for every $\RM_i \in(\mathfrak{R}\setminus \RM^{\mathfrak{R}})$.
\end{proof}

\begin{proposition}\label{Prop:uniqueLM}
Let $\mathcal{C}$ be a satisfiable set of conditionals. Then the ranked union of the elements of $\Mod{\mathcal{C}}$ is the only   $\mpref{\LM}$-minimum element in $\Mod{\mathcal{C}}$.
\end{proposition}

\begin{proof}
It is an  immediate consequence of  Definition \ref{Def:ranked_union} and Proposition \ref{Prop:ranked_union}.
\end{proof}

\LMisRC*

\begin{proof}
It has been proved by Booth and Paris \cite[Theorem 2]{BoothParis1998} that the rational closure of a KB is determined by a  model that is equivalent to the ranked union of $\Mod{\mathcal{C}}$. According to Proposition \ref{Prop:uniqueLM}, the ranked union of $\Mod{\mathcal{C}}$ is the only $\mpref{\LM}$-minimum element in $\Mod{\mathcal{C}}$, that is, the model  $\RM^{\mathrm{rc}}(\mathcal{C})$.
\end{proof}

\section{Proofs of Lemmas~\ref{Lemma:TerminationAndSoundness}, \ref{Lemma:IsModelOfKB} and \ref{Lemma:LayerInclusion}}\label{Proof:Lemmas123}

\subsection{Proof of Lemma~\ref{Lemma:TerminationAndSoundness}}\label{Proof:TerminationAndSoundness}

\TerminationAndSoundness*

\begin{proof}
We show all three simultaneously by complete induction on $i$. So, assume all of Items 1, 2 and 3 hold for all $m < i$. We will show this implies all three hold also for $i$. We assume each $\alpha \in \KB$ is in normal form.
\\
\underline{1. $\states{\KB}{\RM_i} \subseteq \states{\KB}{\RM_{i+1}}$.}
\\
Let $v \in \states{\KB}{\RM_i}$ and let $\alpha \in \KB$ with $\alpha = \bigwedge_{i \leq t} \typ\theta_i \rightarrow (\phi \vee \bigvee_{i \leq s} \typ \psi_i)$ (for some $s,t \geq 0$). We must show $v \in \states{\alpha}{\RM_{i+1}}$. Since $v \in \states{\KB}{\RM_i}$ we know $v \in \states{\alpha}{\RM_i}$. Hence we know that one of the following must hold:
\begin{itemize}
\item
$v \not\in \states{\bullet \theta_k}{\RM_i}$ for some $k$: This means (since $\theta_k$ is propositional) $v$ is not $\pref^{\RM_i}$-minimal in $\states{\theta_k}{\RM_i} = \states{\theta_k}{\RM_{i+1}}$. But then it is also not $\pref^{\RM_{i+1}}$-minimal since, by construction, if $\RM_i(v)\leq\RM_i(w)$ then 
$\RM_{i+1}(v)\leq\RM_{i+1}(w)$. Hence in this case 
$v \not\in \states{\bullet \theta_k}{\RM_{i+1}}$.

\item
$v \in \states{\phi}{\RM_i}$: In this case also $v \in \states{\phi}{\RM_{i+1}}$, since $\states{\phi}{\RM_i} =  \states{\phi}{\RM_{i+1}}$ (because $\phi$ is purely propositional).

\item
$v \in \states{\bullet \psi_k}{\RM_i}$ for some $k$: This means $v$ is $\pref^{\RM_i}$-minimal in $\states{\psi_k}{\RM_i}$. But then it is also $\pref^{\RM_{i+1}}$-minimal, since we assumed $v \in \states{\KB}{\RM_i} = S_{i+1}$, and so by construction of $\RM_{i+1}$
we have that $\RM_{i+1}(w) < \RM_{i+1}(v)$ if and only if $\RM_i(w) < \RM_i(v)$ for all $w \in \U$. Since $\states{\psi_k}{\RM_i} = \states{\psi_k}{\RM_{i+1}}$ (since $\psi_k$ is purely propositional) we obtain that $v$ is $\pref^{\RM_{i+1}}$-minimal in $\states{\psi_k}{\RM_{i+1}}$, \ie, $v \in \states{\bullet \psi_k}{\RM_{i+1}}$.
\end{itemize}
Thus in all three possible cases we obtain $v \in \states{\alpha}{\RM_{i+1}}$ as required.

\noindent{\underline{2. $\RM_i(v_1) < \RM_i(v_2)$ implies $v_1 \in \states{\KB}{\RM_i}$.}}
\\
Suppose $\RM_i(v_1) < \RM_i(v_2)$. Observe that, by construction, this can only be the case if $i>0$. 
Then either $\RM_{i-1}(v_1)< \RM_{i-1}(v_2)$ or $v_2\notin S_i$. 
If $\RM_{i-1}(v_1)<\RM_{i-1}(v_2)$ then, by the inductive hypothesis, $v_1 \in  \states{\KB}{\RM_{i-1}}$, while if $v_2\notin S_i$, then $v_1 \in S_i = \states{\KB}{\RM_{i-1}}$. So either way we get $v_1 \in  \states{\KB}{\RM_{i-1}}$ and so we get the desired conclusion by applying $\states{\KB}{\RM_{i-1}} \subseteq \states{\KB}{\RM_{i}}$ which was just proved in Item~1 above.

\noindent{\underline{3. $\RM_i$ is a ranked interpretation.}}
\\
By construction it immediately follows that $\RM_i$ is a function from $\U$ to $\mathbbm{N}\cup\{\infty\}$.
We need to show the convexity property: if $\RM_i(u)=j$ then, for every $k$ such that $0\leq k<j$, there is a $v\in\U$ for which $\RM_i(v)=k$. If $i=0$, this follows immediately (since $\RM_0(u)=0$ for all $u\in\U$). Otherwise we have by the inductive hypothesis that $\RM_{i-1}$ is a ranked interpretation. We have two cases. (1) $S_{i}=S_{i-1}$: Then $\RM_i=(\RM_{i-1})^{\infty}_{S_{i}}$ from which convexity follows immediately. 
(2) $S_{i}\neq S_{i-1}$: Then $\RM_i=(\RM_{i-1})^1_{S_i}$ from which convexity also follows immediately.


\end{proof}

\subsection{Proof of Lemma~\ref{Lemma:IsModelOfKB}}\label{Proof:IsModelOfKB}

\IsModelOfKB*

\begin{proof}
Let $\RM$ denote $(\RM_i)^{\infty}_{S_i}$. We need to show that for every valuation $v\in\U^{\RM}$, \ie, every $v \in S_i = \states{\KB}{\RM_{i-1}}$, and for every $\alpha \in \KB$, we have $v \in \states{\alpha}{\RM}$. Since $v \in \states{\alpha}{\RM_{i-1}}$ we know one of the following must hold (recalling that $\alpha$ is expressed in normal form $\bigwedge_{i \leq t} \bullet\theta_i \rightarrow (\phi \vee \bigvee_{i \leq s} \bullet \psi_i)$.):
\begin{itemize}
\item $v \not\in \states{\bullet \theta_k}{\RM_{i-1}}$ for some $k$: This means $v$ is not $\pref^{\RM_{i-1}}$-minimal in $\states{\theta_k}{\RM_{i-1}}$. But then it is also not $\pref^{\RM}$-minimal in $\states{\theta_k}{\RM} = \states{\theta_k}{\RM_{i-1}} \cap S_{i}$, since if $w \pref^{\RM_{i-1}} v$ and $w \in \states{\theta_k}{\RM_{i-1}}$, then from the former we know $w \in S_{i}$ by Item~2 of the previous lemma. Hence in this case $v \not\in \states{\bullet \theta_k}{\RM}$.
\item $v \in \states{\phi}{\RM_{i-1}}$: In this case also $v \in \states{\phi}{\RM}$, since $\states{\phi}{\RM} =  \states{\phi}{\RM_{i-1}} \cap S_{i}$ (because $\phi$ is purely propositional).
\item $v \in \states{\bullet \psi_k}{\RM_{i-1}}$ for some $k$: This means $v$ is $\pref^{\RM_{i-1}}$-minimal in $\states{\psi_k}{\RM_{i-1}}$. But then it is also $\pref^{\RM}$-minimal in $\states{\psi_k}{\RM} =  \states{\psi_k}{\RM_{i-1}} \cap S_{i}$, since $\pref^{\RM_{i-1}} \subseteq \pref^{\RM}$. Hence $v \in \states{\bullet \psi_k}{\RM}$.
\end{itemize}
Thus in all three possible cases we obtain $v \in \states{\alpha}{\RM}$ as required.
\end{proof}

\subsection{Proof of Lemma~\ref{Lemma:LayerInclusion}}\label{Proof:LayerInclusion}

\LayerInclusion*

\begin{proof}
Let $v \in M_i$. By construction, $S_i = \states{\KB}{\RM_{i-1}}$ where $\RM_{i-1} = (L_0, \ldots, L_{i-1}, (\U \setminus \bigcup_{j<i} L_j,\emptyset)$. Let $\alpha \in \KB$, with $\alpha = \bigwedge_{i \leq t} \typ\theta_i \rightarrow (\phi \vee \bigvee_{i \leq s} \typ \psi_i)$ (for some $s,t \geq 0$). We must show $v$ satisfies $\alpha$ in $\RM_{i-1}$, so assume $v$ satisfies $\neg\phi$ and is a minimal $\theta_k$-state in $\RM_{i-1}$ for all $k$. We must show $v$ is a minimal $\psi_k$-state in $\RM_{i-1}$ for at least one $k$. Since we assume $M_j = L_j$ for all $j<i$, we have $\RM_{i-1} = (M_0, \ldots, M_{i-1}, (\U \setminus \bigcup_{j<i} M_j),\emptyset)$. Since $v \in M_i$, we can show that, for {\em any} propositional sentence $\lambda$, we have that $v$ is a minimal $\lambda$-state in $(M_0, \ldots, M_{i-1}, (\U \setminus \bigcup_{j<i} M_j),\emptyset)$ if and only if it is a minimal $\lambda$-state in $(M_0,\ldots, M_i,\emptyset)$. Thus, from the fact that $(M_0,\ldots, M_i,\emptyset)$ is a ranked model of $\KB$, we obtain our conclusion.
\end{proof}

\end{appendices}

\end{document}